\documentclass[transmag]{IEEEtran}
\usepackage{latexsym}
\def\BibTeX{{\rm B\kern-.05em{\sc i\kern-.025em b}\kern-.08em T\kern-.1667em\lower.7ex\hbox{E}\kern-.125emX}}

\input{mysymbol.sty}

\usepackage[utf8]{inputenc} 
\usepackage[T1]{fontenc}    
\usepackage{hyperref}       
\usepackage{url}            
\usepackage{booktabs}       
\usepackage{amsfonts}       
\usepackage{nicefrac}       
\usepackage{microtype}      
\usepackage{lipsum}
\usepackage{amsmath}
\usepackage{amssymb}        
\usepackage{amsthm} 
\theoremstyle{plain}

\theoremstyle{definition}

\usepackage{verbatim}
\usepackage{empheq,float}
\usepackage{graphicx,caption,subcaption,balance}
\usepackage{stfloats}
\usepackage[dvipsnames]{xcolor}
\usepackage{epstopdf}
\usepackage{algorithm}
\usepackage{algpseudocode}
\usepackage{mathrsfs}
\usepackage{psfrag}
\usepackage{enumitem}

\usepackage{amsthm}
\newtheorem{theorem}{Theorem}
\newtheorem{claim}[theorem]{Claim}
\newtheorem{prop}{Proposition}
\newtheorem{remark}{\hspace{0pt}\bf Remark}
\newtheorem{assumption}{\textit{Assumption}}

\def \vec {\operatorname{vec}}

\def \vechh {\operatorname{vechh}}

\begin{document}

\title{Learning Time-Varying Graphs from Online Data}

\author{Alberto~Natali,~\IEEEmembership{Student Member,~IEEE,}
       Elvin~Isufi,
       Mario~Coutino,~\IEEEmembership{Member,~IEEE,}
        and~Geert~Leus,~\IEEEmembership{Fellow,~IEEE}
\thanks{A. Natali, E. Isufi and G. Leus are with the Faculty of Electrical Engineering, Mathematics and Computer Science, Delft University of Technology, Delft, 2628 CD, The Netherlands (e-mail: \{a.natali; e.isufi-1; g.j.t.leus\}@tudelft.nl). M. Coutino is with Radar Technology, TNO, The Hague, The Netherlands (e-mail: m.a.coutinominguez@tudelft.nl).}
\thanks{Part of this paper was presented in~\cite{natali2021online}.}}

\IEEEtitleabstractindextext{\begin{abstract}This work proposes an algorithmic framework to learn time-varying graphs from online data. The generality offered by the framework renders it model-independent, i.e., it can be theoretically analyzed in its abstract formulation and then instantiated under a variety of model-dependent graph learning problems. This is possible by phrasing  (time-varying) graph learning as a composite optimization problem, where different functions regulate different desiderata, e.g., data fidelity, sparsity or smoothness. Instrumental for the findings is recognizing that the dependence of the majority (if not all) data-driven graph learning algorithms on the data is exerted through the empirical covariance matrix, representing a sufficient statistic for the estimation problem. Its user-defined recursive update enables the framework to work in non-stationary environments, while iterative algorithms building on novel time-varying optimization tools explicitly take into account the temporal dynamics, speeding up convergence and implicitly including a  temporal-regularization of the solution. We specialize the framework to three well-known graph learning models, namely, the Gaussian graphical model (GGM), the structural equation model (SEM), and the smoothness-based model (SBM), where we also introduce ad-hoc vectorization schemes for structured matrices (symmetric, hollows, etc.) which are crucial to perform correct gradient computations, other than enabling to work in low-dimensional vector spaces and hence easing storage requirements. After discussing the theoretical guarantees of the proposed framework, we corroborate it with extensive numerical tests in synthetic and real data.\end{abstract}

}

\maketitle

\section{Introduction}

\IEEEPARstart{L}{earning} network topologies from data is very appealing. On the \textit{interpretable} side, the structure of a network reveals important descriptors of the network itself, providing to humans a prompt and explainable decision support system; on the \textit{operative} side, it is a requirement for processing and learning architectures operating on graph data, such as graph filters~\cite{coutino2019advances}. When this structure is not readily available from the application, a fundamental question is how to \textit{learn} it from data. The class of problems and the associated techniques concerning the identification of a network structure (from data)  are known as graph topology identification (GTI), graph learning, or network topology inference~\cite{mateos2019connecting, dong2019learning}. 

While up to recent years the GTI  problem has been focused on learning \textit{static} networks, i.e., networks which do not change their structure over time, the pervasiveness of networks with a \textit{time-varying} component  has quickly demanded new learning paradigms. This is the case for biological networks~\cite{kim2014inference}, subject to changes due to genetic and environmental factors, or financial markets~\cite{mantegna1999hierarchical}, subject to changes due to political  factors, among others. In these scenarios, a static approach would fail in accounting for the temporal variability of the underlying  structure, which is strategic  to, e.g., detect anomalies or discover new emerging communities.

In addition, prior (full) data availability should not be considered as a given. In real time applications, data need to be processed on-the-fly with low latency to, e.g., identify and block cyber-attacks in a communication infrastructure, or fraudulent transactions in a financial network. Thus, another learning component to take into account, is the modality of data acquisition. Here, we consider the extreme case in which data are processed on-the-fly, i.e., a fully \textit{online}  scenario.  

It is then clear how the necessity of having algorithms to learn time-varying topologies from online data is motivated by physical scenarios. For clarity, we elaborate on the three keywords -  identification, time-varying and online - which constitute, other than the title of the present work, also its main pillars.
\begin{itemize}
    \item \textit{Identification/learning}: it refers to the (optimization) process of learning the graph topology. 
    \item \textit{Time-Varying/dynamic}: it refers to the temporal variability of the graph in its edges, in opposition to the static case.
    \item \textit{Online/streaming}: it refers to the modality in which the data arrive and/or are processed, in opposition to a batch approach which makes use of the entire bulk of data. 
\end{itemize}
This emphasis on the terminology is important to understand the differences between the different existing works, presented next.

\subsection{Related Works}
\label{sec:related-works}
Static GTI has been originally addressed  from a statistical viewpoint and only in the past decade under a graph signal processing (GSP) framework~\cite{sandryhaila2013discrete}, in which different assumptions are made on how the data are coupled with the unknown topology; see \cite{mateos2019connecting, dong2019learning} for a tutorial. Only  recently, dynamic versions of the static counterparts have been proposed. For instance, \cite{kalofolias2017learningtvgraphs, yamada2020time} learn a sequence of graphs by enforcing a prior (smoothness or sparsity) on the edges of consecutive graphs; similarly, the work in \cite{hallac2017network} extends the graphical Lasso \cite{friedman2008sparse} to account for the temporal variability, i.e., by estimating a sparse time-varying precision matrix.  In addition to these works, the inference of causal relationships in the network structure, i.e., directed edges, has been considered in~\cite{baingana2017cascades, money2021online}. See~\cite{giannakis2018topology} for a  review of dynamic topology inference approaches.

The mentioned approaches tackle the dynamic graph learning problem by means of a two-step approach: \textit{i)} first, all the samples are collected and split into possibly overlapping windows;  \textit{ii)} only then the topology associated to each window is inferred from the data, possibly constrained to be similar to the adjacent ones. This modus-operandi fails to address the \textit{online} (data-streaming) setting, where data have to be processed on-the-fly either due to architectural (memory, processing) limitations or (low latency) application requirements, such as real-time decision making. 

 This line of work has been freshly investigated by~\cite{vlaski2018online}, which considers signals evolving according to a heat diffusion process, and by~\cite{shafipour2020online}, which assumes the data are graph stationary~\cite{marques2017stationary}. In~\cite{zaman2020online}, the authors consider a vector autoregressive model  to learn causality graphs by exploiting the temporal dependencies, while~\cite{saboksayr2021online} proposes an online task-dependent (classification) graph learning algorithm, in which class-specific graphs are learned from labeled signals (training phase) and then used to classify new unseen data. 

Differently from these works, our goal here is to provide a general (model-independent) algorithmic framework for time-varying GTI from online data that can be specialized to a variety of static graph learning problems. In particular, the generalization given by the framework enables us to render a static graph learning problem into its time-varying counterpart and to solve it via novel time-varying optimization techniques~\cite{simonetto2020time},  providing a trade off between the solution accuracy and the velocity of execution. We introduce ad-hoc vectorization schemes for structured matrices to solve graph learning problems in the context of the Gaussian graphical model, the structural equation model, and the smoothness based model.
All in all, a mature time-varying GTI  framework for online data is yet to be conceived. This is our attempt to pave the way for a unified and general view of the problem, together with solutions to solve it.

\subsection{Contributions}

This paper proposes a general-purpose algorithmic blueprint which unifies the theory of learning time-varying graphs from online data. The specific contributions of this general framework are:
    \begin{enumerate}
        \item[a)] it is \textit{model-independent}, i.e., it can be analyzed in its abstract form and then specialized under different graph learning models. We show how to instantiate three such models, namely, the Gaussian graphical model (GGM), the  structural equation model (SEM) and the smoothness-based model (SBM);
        
        \item[b)] it operates in \textit{non-stationary} environments, i.e., when the data statistics change over time. This is possible by expressing the considered models in terms of the sample covariance matrix, which can be then updated recursively for each new streaming sample with a user-defined function, which discards past information.
        
        \item[c)] it is \textit{accelerated} through a prediction-correction strategy, which takes into account the time-dimension. Its iterative nature enables a trade-off between following the optimal solution (accuracy) and an  approximate solution (velocity). It also exhibits an implicit regularization of the cost function due to the limited iteration budget at each time-instant, i.e., similar solutions at closed time instants are obtained.
    \end{enumerate}
    
    \noindent\textit{Notation}:  we use $x(i)$ and $X(i,j)$ to denote the $i$-th entry of the column vector $\bbx$ and the $ij$-th entry of the matrix $\bbX$, respectively.  Superscripts $^\top$ and $^\dagger$ denote the transpose and the pseudoinverse of a matrix, respectively, while operators $\tr(\cdot)$ and $\vec(\cdot)$ denote the matrix trace and matrix vectorization, respectively. The vectors $\bb0$ and $\boldsymbol{1}$, and the matrix  $\bbI$, denote the all-zeros vector, the all-ones vector, and the identity matrix, with dimension clarified in the context. The operators $\otimes$, $\odot$, $\oslash$ and $^\circ$ stand for Kronecker product, Hadamard (entry-wise) product, Hadamard (entry-wise) division and Hadamard (entry-wise) power, respectively.  We have $[\; \cdot \; ]_{+}= \max(\bb0, \cdot)$, where the maximum operates in an entry-wise fashion. Also, $\iota_{\ccalX}(\cdot)$ is the indicator function for the convex set $\ccalX$, for which holds $\iota_{\ccalX}(\bbx)= 0$ if $\bbx \in \ccalX$ and $+ \infty$ otherwise. Given two functions $f(\cdot)$ and $g(\cdot)$, $f \circ g(\cdot)$ denotes their composition. A function $f(\cdot)$ with argument $\bbx \in \reals^N$, which is parametrized by the time $t$, is denoted by $f(\bbx;t)$.  The gradient of the function $f(\bbx;t)$ with respect to $\bbx$ at the point $(\bbx;t)$ is denoted with $\nabla_\bbx f(\bbx;t)$, while $\nabla_{\bbx \bbx} f(\bbx;t)$ denotes the Hessian evaluated at the same point. The time derivative of the gradient, denoted with $\nabla_{t \bbx} f(\bbx;t)$, is the partial derivative of $\nabla_\bbx f(\bbx;t)$ with respect to the time $t$, i.e., the mixed first-order partial derivative vector of the objective. Finally, $\|\cdot\|_p$ denotes the $\ell_p$ norm of a vector or, for a matrix, the $\ell_p$ norm of its vectorization. The Frobenius norm of a matrix is denoted with $\|\cdot\|_F$. Without any subscript, the  norm $\|\cdot\|$ indicates the spectral norm.

\section{Problem Formulation}
In this section, we formalize the problem of learning graphs from data. In Section~\ref{sec:s-GTI}, we introduce the static graph topology inference problem, where we also recall three well-known models from the literature. Then, in Section~\ref{sec: online-TV} we formulate the (online) dynamic graph topology inference problem.

\subsection{Graph Topology Identification}
\label{sec:s-GTI}

We consider data living in a non-Euclidean domain described by a  graph $\ccalG=\{\ccalV, \ccalE, \bbS\}$, where $\ccalV=\{1, \ldots, N\}$ is the vertex set, $\ccalE \subseteq \ccalV \times \ccalV$ is the edge set, and $\bbS$ is an $N \times N$ matrix encoding the topology of the graph. The matrix $\bbS$ is  referred to as the graph shift operator (GSO) and typical instantiations include the (weighted) adjacency matrix $\bbW$~\cite{sandryhaila2013discrete} and the  graph Laplacian $\bbL$~\cite{shuman2013emerging}. By associating to each node $i \in \ccalV$ a scalar value $x(i)$, we define $\bbx=[x(1), \ldots, x(N)]^\top \in \mathbb{R}^N$ as a \textit{graph signal} mapping the node set to the set of real numbers.

Consider now the matrix $\bbX=[\bbx_1, \ldots, \bbx_T]$ that stacks over the columns $T$ graph signals generated from an unknown  graph-dependent process $\ccalF(\cdot)$; i.e., $\bbX= \ccalF(\bbS)$. Then, a GTI algorithm aims to learn the graph topology, i.e., to solve the ``inverse'' problem (not always well defined):
\begin{align}
\label{eq:inverse}
    \bbS= \ccalF^{-1}(\bbX).
\end{align}
The function ${\mathcal F}(\cdot)$ basically describes how the data are coupled with the graph and its knowledge is crucial. The data and the graph alone are insufficient to cast a meaningful graph learning problem. On one side, we need to know how the data depends on the graph from which they are generated. On the other side, we have to enforce some prior knowledge on the graph we want to learn.

\smallskip\noindent\textbf{Graph-data models.}
The choice of a data model is the forerunner of any GTI technique and, together with the graph-data coupling priors (e.g., smoothness, bandlimitedness) differentiates the different  approaches. Due to their relevance for this work, we  recall three widely used topology identification methods, namely the Gaussian graphical model~\cite{dempster1972covariance}, the structural equation model~\cite{ullman2003structural}, and the smoothness-based model~\cite{kalofolias2016learn}.

\makeatletter
\renewcommand\subsubsection{\@startsection{subsubsection}{3}{\z@}%
                                     {-1ex\@plus -1ex \@minus -.2ex}%
                                     {-0.5ex \@plus -.2ex}
                                     {\normalfont\normalsize\itshape}}
\makeatother
\subsubsection{Gaussian graphical model (GGM)} assumes each graph signal $\bbx_t$ is drawn from a multivariate Gaussian distribution $\ccalN(\boldsymbol{\mu}, \mathbf{\Sigma})$ with mean $\boldsymbol{\mu}$ and positive-definite covariance matrix $\mathbf{\Sigma}$. 
By setting the graph shift operator to be the precision matrix  $\bbS=\mathbf{\Sigma}^{-1}$, \emph{graph learning in a GGM amounts to precision matrix estimation}, which in a maximum likelihood  (MLE) sense can be formulated as:
\begin{align}
\label{eq:ggm}   
\begin{array}{l}
\underset{\bbS}{\operatorname{minimize}} -\log \operatorname{det}(\bbS) +\operatorname{tr}(\bbS \hat{\mathbf{\Sigma}}) \\
\;\;\; \text { s. t. } \quad \bbS \in \mathbb{S}_{++}^{N}
\end{array}
\end{align}
where $\empcov=  \frac{1}{T} \bbX\bbX^\top$ is the sample covariance matrix  and  $\mathbb{S}_{++}^{N}$ is the convex cone of positive-definite matrices.
In this context, matrix $\bbS$ can be interpreted as the adjacency matrix  (with self loops), although the problem can also be solved under some additional constraints forcing ${\bf S}$ to be a Laplacian~\cite{kumar2020unified}.

\subsubsection{Structural equation model (SEM)} neglecting
possible external inputs, and assuming an undirected graph, the SEM poses a linear dependence between the signal value $x_t(i)$ at node $i$ and the signal values at some other nodes $\{x_t(j)\}_{j\neq i}$, representing the endogenous variables, i.e.,:
\begin{align}
\label{eq:sem-model}
    x_t(i)= \sum_{j\neq i } S(i,j)  x_t(j) + e_t(i), \quad t=1, \ldots, T
\end{align}
where $S(i,j)$ weights the influence that node $j$ exerts on node $i$, and $e_t(i)$ represents unmodeled effects. In this view, with $\bbS$ encoding the graph connectivity, model \eqref{eq:sem-model} considers each node to be influenced only by its one-hop neighbors. In vector form, we can write \eqref{eq:sem-model}  as:
\begin{align}
\label{eq:sem-data-model}
    \bbx_t= \bbS\bbx_t + \bbe_t, \;\;\; t=1, \ldots, T,
\end{align}
with $S(i,i)=0$, for $i=1, 2, \ldots, N$. Also, we consider $\bbe_t$ white noise with standard deviation $\sigma_e$.
Graph learning under a SEM implies estimating matrix $\bbS$ by solving:
\begin{align}
\label{eq:sem-estimation}
\begin{array}{l}
\underset{\bbS}{\operatorname{minimize}}\;\;\frac{1}{2T} \|\bbX - \bbS\bbX\|_F^2 + g(\bbS), \\
\;\;\; \text { s. t. } \;\;\bbS \in \ccalS 
\end{array}
\end{align}
where $\ccalS = \{\bbS \vert \operatorname{diag}(\mathbf{S})=\bb0,  S(i,j)=S(j,i), i \neq j\}$, and $g(\bbS)$ is a regularizer enforcing $\bbS$ to have specific properties; e.g., sparsity.
In this context, matrix $\bbS$ is usually interpreted as the adjacency matrix of the network (without self loops). The first term of \eqref{eq:sem-estimation} can be equivalently rewritten as:
\begin{align}
    f(\bbS)\!=\!\frac{1}{2T} \|\bbX - \bbS\bbX\|_F^2\!=\! \frac{1}{2} \!  [\tr(\bbS^2\empcov)\!   -\!2  \tr(\bbS\empcov) \! +\! \!\tr(\empcov)].
\end{align}
which highlights its dependence on $\empcov$.

\subsubsection{Smoothness-based model (SBM)} assumes each graph signal $\bbx_t$ to be smooth over the graph $\ccalG$, where the notion of graph-smoothness is formally captured  by the Laplacian quadratic form:
\begin{align}
    \text{LQ}_{\ccalG}(\bbx_t):= \bbx_t^\top \bbL \bbx_t = \sum_{i \neq j} W(i,j) (x_t(i)- x_t(j))^2.
\end{align}
A low value of LQ$_{\ccalG}(\bbx_t)$ suggests that adjacent nodes $i$ and $j$ have similar values $x_t(i)$ and $x_t(j)$ when the edge weight $W(i,j)$ is high. 

Thus, the quantity:
\begin{align}
\label{eq:TV}
    \overline{\text{LQ}}_{\ccalG}(\bbX)= \frac{1}{T}\sum_{t=1}^T  \text{LQ}_{\ccalG}(\bbx_t)= \frac{1}{T}\tr(\bbX^\top\bbL\bbX)=\tr( {\bf L} \hat{\boldsymbol \Sigma}) 
\end{align}
represents the average signal smoothness on top of $\ccalG$, which can be rewritten as the graph-dependent function:
\begin{align}
\label{eq:smoothness-quantity}
    f(\bbS) = \tr(\Diag(\bbS \boldsymbol{1})\empcov ) - \tr(\bbS\empcov)
\end{align}
with $\bbS= \bbW$. Building upon this quantity, graph learning under a graph smoothness prior can be casted as:
\begin{align}
\label{eq:smoothness-estimation}
    \begin{array}{l}
\underset{\bbS}{\operatorname{minimize}} \;\;\;f(\bbS) +g(\mathbf{S}) \\
\text { s. t. } \bbS \in \ccalS
\end{array}
\end{align}

where the term $g(\bbS)$ accommodates for additional topological properties (e.g., sparsity) and also helps avoiding the trivial solution $\bbS=\bb0$. The set $\ccalS=\{\bbS \vert \operatorname{diag}(\mathbf{S})=\bb0, S(i,j)=S(j,i) \geq 0, i \neq j\}$ encodes the  topological structure, which coincides with the set of hollow  symmetric matrices (i.e., with zeros on the diagonal) with positive entries.

\begin{remark}
In~\cite{kalofolias2016learn}, the authors express the smoothness quantity \eqref{eq:TV} in terms of the weighted adjacency matrix $\bbW$ and a matrix  $\bbZ \in \mathbb{R}_+^{N \times N}$ representing the row-wise (squared) Euclidean distance matrix of $\bbX$; i.e.,  $\tr(\bbX^\top\bbL\bbX)= \frac{1}{2} \tr(\bbW \bbZ)= \frac{1}{2}\|\bbW \odot \bbZ\|_1$. This formulation mainly brings the intuition that adding explicitly a sparsity term to the objective function would simply add a constant term to $\bbZ$. We favour \eqref{eq:smoothness-quantity} as a measure of graph signal smoothness since it fits within our framework, as will be clear soon. We emphasize however how the two formulations are equivalent, since $\empcov$ can be directly expressed as a function of $\bbZ$.

\end{remark}

\subsection{Online Time-Varying Topology Identification}
\label{sec: online-TV}

When the graph topology changes over time, the changing interactions are represented by the sequence of graphs $\{\ccalG_t=\{\ccalV, \ccalE_t, \bbS_t\}\}_{t=1}^\infty$, where $t \in {\mathbb N}_+$ is a discrete time index.   This sequence of graphs, which is discrete in nature, can be interpreted as the sampling of some "virtual" continuous time-varying graph using the sampling period $h=1$.  To relate our expressions to existing literature, we will make the parameter $h$ explicit in the formulas, yet it is important to remember that $h=1$.
Together with the graph sequence $\{\ccalG_t\}_{t=1}^\infty$, we consider also streaming graph signals $\{\bbx_t\}_{t=1}^\infty$, such that signal $\bbx_t$ is associated to graph $\ccalG_t$. At this point, we are ready to formalize  the time-varying graph topology identification (TV-GTI) problem.
\vspace{0.2cm}

\noindent \textbf{Problem statement.} \textit{Given an online sequence of graph signals $\{\bbx_t\}_{t=1}^{\infty}$ arising from an unknown time-varying network, the goal is to identify the time-varying graph topology $\{\ccalG_t\}_{t=1}^{\infty}$; i.e., to learn the graph shift operator sequence $\{\bbS_t\}_{t=1}^\infty$ from $\{\bbx_t\}_{t=1}^{\infty}$. On top of this, to highlight the trade-off between accuracy and low-latency of the algorithm's solution.}

\vspace{0.2cm}
Mathematically, our goal is to solve the sequence of time-invariant problems:
\begin{align}
\label{eq:probl-stat}
    \bbS_t^\star:= \argmin_\bbS F(\bbS;t) \quad t=1, 2, \ldots
\end{align}
where function $F(\cdot;t)$ is a time-varying cost function that depends on the data model [cf. Section~\ref{sec:s-GTI}], and the index $t$ makes the dependence on time explicit, which is due to the arrival of new data.
Although we can solve problem \eqref{eq:probl-stat} for each $t$ separately with (static) convex optimization tools, the need of a low-latency stream of solutions makes this strategy unappealing. This approach also fails to capture the inherent temporal structure of the problem, i.e, it does not exploit the prior time-dependent structure of the graph, which is necessary in time-critical applications.

To exploit also this temporal information, we build on recent advances of time-varying optimization \cite{simonetto2016class,simonetto2020time} and propose a general framework for TV-GTI suitable for non-stationary environments. The proposed approach operates on-the-fly and updates the solution as a new signal $\bbx_t$ becomes available. 
The generality of this formulation enables us to define a \textit{template} for the TV-GTI problem, which can be specialized to a variety of static GTI methods. The only information required is the first-order (gradient) and possibly second-order (Hessian) terms of the function. In the next section, we lay down the mathematics of the proposed approach.  The central idea is to follow the optimal time-varying solution of problem \eqref{eq:probl-stat} with lightweight proximal operations \cite{martinet1970regularisation}, which can be additionally accelerated  with a  \textit{prediction-correction} strategy. This strategy, differently from other adaptive optimization strategies such as least mean squares and recursive least squares, uses an evolution model to predict the solution, and observes new data to correct the predictions. The considerations of Section~\ref{sec:dynamic-graph-learning} will be then specialized to the different data models of Section~\ref{sec:s-GTI} in Section~\ref{sec:network-models}, further analyzed theoretically in Section~\ref{sec:convergence-analysis}, and finally validated experimentally in Section~\ref{sec:numerical-result}.

\section{Online Dynamic Graph Learning}
\label{sec:dynamic-graph-learning}

To maintain our discussion general, we consider the \textit{composite}  time-varying function:
\begin{align}
\label{eq: composite}
    F(\bbS;t):= f(\bbS;t) + \lambda g(\bbS;t)
\end{align}
where $f: \mathbb{R}^{N \times N} \times {\mathbb N}_+ \rightarrow {\mathbb R}$ is a smooth\footnote{We use the term smoothness for functions and the term graph-smoothness for graph signals.} strongly convex function \cite{vial1983strong}  encoding a fidelity measure and  $g: \mathbb{R}^{N \times N} \times {\mathbb N}_+ \rightarrow {\mathbb R}$ is a closed convex and proper function, potentially non differentiable, representing possible regularization terms. For instance, function $f(\cdot)$ can be the GGM objective function of \eqref{eq:ggm}, the SEM least-squares term of \eqref{eq:sem-estimation}, or the SBM smoothness measure in~\eqref{eq:TV}.

Solving a time-varying optimization problem implies solving the \textit{template} problem:
\begin{align}
\label{eq:time-varying}
    \bbS_t^\star:= \argmin_\bbS  f(\bbS;t) + \lambda g(\bbS;t) \quad \text{for } t=1, 2, \ldots
\end{align}
In other words, the goal is to find the sequence of optimal solutions $\{\bbS_t^\star\}_{t=1}^\infty$ of \eqref{eq:time-varying}, which we will also call the \textit{optimal trajectory}. However, solving exactly problem \eqref{eq:time-varying} in real time is infeasible because of the computational and time constraints. The exact solution may also be unnecessary since by itself it still approximates the true underlying time-varying graph. Under these considerations, an online algorithm that updates the approximate solution $\hat{\bbS}_{t+1}$ of \eqref{eq:time-varying} at time $t+1$, based on the former (approximate) solution $\hat{\bbS}_{t}$ is highly desirable for low complexity and fast execution\footnote{Problem \eqref{eq:time-varying} also endows the constrained case, in which the function $g(\cdot)$ comprises indicator functions associated to each constraint.}.

\subsection{Reduction} 
\label{subsec:reduction}
Instrumental for the upcoming analysis is to observe that the number of independent variables of the graph representation matrix plays an important role in terms of storage requirements, processing complexity and, most importantly, in the correct computations of function derivatives with respect to those variables. Thus, when considering structured matrices, such as symmetric, hollow or  diagonal, we need to take into account their structure. We achieve this by  ad-hoc vectorization schemes through  duplication and elimination matrices, inspired by~\cite{magnus2019matrix}.

Consider a matrix $\bbS \in \reals^{N \times N}$ and its corresponding ``standard'' vectorization $ \vec(\bbS) \in \reals^{N^2}$. Depending on the specific structure of  $\bbS$, different  reduction and vectorization schemes can be adopted, leading to a lift from a matrix space to a vector space. The following spaces are of interest.

\smallskip\noindent\textbf{h-space.} \label{subsec:h-space} If $\bbS$ is symmetric, the number of independent variables is $k=N(N+1)/2$, i.e., the variables in its diagonal and its lower (equivalently, upper) triangular part. We can isolate these variables by representing  matrix $\bbS$ with its \textit{half-vectorization} form, which we denote as  $\bbs=\operatorname{vech}(\bbS) \in \mathbb{R}^{k}$. This isolation is possible by introducing the elimination matrix $\bbE \in \mathbb{R}^{k \times N^2}$ and the duplication matrix $\bbD \in \mathbb{R}^{N^2 \times k}$ which respectively selects the independent entries of $\bbS$, i.e., $\bbE\operatorname{vec}(\bbS) =\bbs$, and duplicates the entries of $\bbs$, i.e, $\bbD\bbs = \operatorname{vec}(\bbS)$.  We call this vector space as the half-vectorization space (h-space).
    
\smallskip\noindent\textbf{hh-space.} If $\bbS$ is symmetric and hollow, the number of independent variables is $l=N(N-1)/2$, i.e.,  the variables on its strictly lower (equivalently, upper) triangular part.  In this case, we can represent  matrix $\bbS$ in its \textit{hollow half-vectorization} form, which we denote as $\bbs=\vechh(\bbS) \in \mathbb{R}^l$. This reduction is achieved by applying the hollow elimination and duplication matrices $\bbE_h \in \mathbb{R}^{l \times N^2}$ and $\bbD_h \in \mathbb{R}^{N^2  \times l}$, respectively, to the vectorization of $\bbS$. In particular, $\bbE_h$ extracts the variables of the strictly lower  triangular part of the matrix, i.e., $\bbs = \bbE_h \vec(\bbS)$,  while $\bbD_h$ duplicates the values and fills in zeros in the correct positions, i.e., $\vec(\bbS) = \bbD_h \bbs$. We refer to the associated vector space as the hollow half-vectorization space (hh-space).

\smallskip\noindent With the above discussion in place, we can now illustrate the general framework in terms of vector-dependent functions $f(\bbs)$ for a vector $\bbs$, in contrast to matrix-dependent functions $f(\bbS)$, simplifying exposition and notation. However, we underline that the information embodied in $\bbS$ and $\bbs$ is the same. 

\subsection{Framework} We develop a prediction-correction strategy for problem  \eqref{eq:time-varying} that starts from an estimate  $\hat{\bbs}_t$ at time instant $t$, and \textit{predicts} how this solution will change in the next time step $t+1$. This predicted topology is then \textit{corrected} after a new datum $\bbx_{t+1}$ is available at time $t+1$. More specifically, the scheme has the following two steps:

\begin{enumerate}[labelsep=0.1cm, leftmargin=*]
    \item[\textbf{(1)}] \hspace{-0.1cm} \textit{Prediction}: at time $t$, an approximate function $\hat{F}(\bbs;t+1)$ of the true yet \textit{unobserved} function $F(\bbs;t+1)$ is formed, using only information available at time $t$. 
    Then, using this approximated cost, we derive an estimate  $\bbs_{t+1|t}^\star$, of how the topology will be at time $t+1$, using only the information up to time $t$. This estimate is found by solving:
   \begin{align}
    \label{eq:prediction}
        \bbs_{t+1\vert t}^\star:= \argmin_\bbs \hat{F}(\bbs;t+1).
    \end{align}
   To avoid solving \eqref{eq:prediction} for each $t$, we find an estimate $\hat{\bbs}_{t+1\vert t}$  by applying $P$ iterations of a problem-specific descent operator $\hat{\ccalT}$ (e.g., gradient descent, proximal gradient) for which $\bbs_{t+1\vert t}^\star= \hat{\ccalT} \bbs_{t+1\vert t}^\star$, i.e., $\bbs_{t+1\vert t}^\star$ is a fixed point of $\hat{\ccalT}$.  See Appendix~\ref{sec:app-a} for possible instances of $\hat{\ccalT}$.
    
    In other words, problem \eqref{eq:prediction} is solved recursively as:
    \begin{align}
    \label{eq:prediction-step}
    \hat{\bbs}^{p+1}= \hat{\ccalT}\hat{\bbs}^p, \quad  p=0,1, \ldots, P-1
\end{align}
with $\hat{\bf s}^0=\hat{\bf s}_t$.
Once $P$ steps are performed, the predicted topology is set to $\hat{\bbs}_{t+1 \mid t}= \hat{\bbs}^{P}$, which approximates the solution of \eqref{eq:prediction} and, in turn, will be close to $\bbs_{t+1}^{\star}$ at time $t+1$.

For our framework, we consider a Taylor-expansion based prediction to approximate the first term of $F(\cdot; t+1)$, i.e.,    $f(\cdot; t+1)$ [cf. \eqref{eq: composite}], leading to the following quadratic function: 
    \begin{align}
    \label{eq:approximate-function}
     \nonumber \hat{f}(\bbs; t+1)&= \frac{1}{2} \bbs^\top \nabla_{\bbs\bbs} f\left(\hat{\bbs}_{t} ; t\right) \bbs
 +\big[\nabla_{\bbs} f\left(\hat{\bbs}_{t} ; t\right)+ \\
 &+ h \nabla_{t \bbs} f\left(\hat{\bbs}_{t} ; t\right)-\nabla_{\bbs \bbs} f\left(\hat{\bbs}_{t} ; t\right) \hat{\bbs}_{t}\big]^\top\bbs
 \end{align}
    where $\nabla_{\bbs \bbs} f (\cdot) \in \mathbb{R}^{N \times N}$ is the Hessian matrix of $f(\cdot)$ with respect to $\bbs$ and  $\nabla_{t \bbs} f(\cdot) \in \mathbb{R}^{N}$ is the partial derivative of the gradient of $f(\cdot)$ w.r.t.  time $t$. 
    
   To approximate the second term of $F(\cdot; t+1)$, i.e., $g(\cdot; t+1)$ [cf. \eqref{eq: composite}], we use a one step-back prediction, i.e., ${\hat{g}(\bbs; t+1)= g(\bbs; t)}$. This implies that $\hat{g}(\cdot)$ does not depend on $t$, which in turn makes the constraint set and the regularization term independent of time, an assumption usually met in state-of-the-art topology identification~\cite{mateos2019connecting}. Henceforth, we will omit this time dependency.

    \item[\textbf{(2)}] \textit{Correction}: at time $t+1$ the new data $\bbx_{t+1}$ and hence the cost function $F(\bbs;t+1)$ becomes available. Thus, we correct the prediction $\hat{\bbs}_{t+1\vert t}$  by solving the correction problem:
   \begin{align}
    \label{eq:correction}
        \bbs_{t+1}^\star:= \argmin_\bbs F(\bbs;t+1).
    \end{align}
    Also in this case, we solve \eqref{eq:correction} with iterative methods to obtain an approximate solution $\hat{\bbs}_{t+1}$ by applying $C$ iterations of an operator $\ccalT$. In other words, the correction problem \eqref{eq:correction} is addressed through the recursion:
    \begin{align}
    \label{eq:correction-step}
    \hat{\bbs}^{c+1}= \ccalT\hat{\bbs}^c, \quad c=0,1, \ldots, C-1
\end{align}
with  $\hat{\bf s}^0=\hat{\bf s}_{t+1|t}$. Once the $C$ steps are performed, the correction graph $\hat{\bbs}_{t+1}$ is set to $\hat{\bbs}_{t+1}= \hat{\bbs}^{C}$, which will approximate the solution $\bbs_{t+1}^{\star}$ of \eqref{eq:correction}.
\end{enumerate}
%
%

\begin{algorithm}[t]
\begin{algorithmic}[1]
\Require Feasible $\hat{\bbS}_0$, $f(\bbS; t_0)$, $P$, $C$, operators $\hat{\ccalT}$ and  $\ccalT$ 
\State $\hat{\bbs}_0 \xleftarrow[]{}$ ad-hoc vectorization of $\hat{\bbS}_0$
\For{$t=0,1, \ldots$}
\State // \textit{Prediction}
\State Initialize the predicted variable $\hat{\bbs}^0= \hat{\bbs}_t$

\For{$p= 0,1, \ldots, P-1$}

    Predict $\hat{\bbs}^{p+1}$ with \eqref{eq:prediction-step}
\EndFor

Set the predicted variable $ \hat{\bbs}_{t+1 \mid t}= \hat{\bbs}^{P}.
$
\State // \textit{Correction - time $t+1$: new data arrive}
\State Initialize the corrected variable $\hat{\bbs}^0= \hat{\bbs}_{t+1 \mid t}$
\For{$c= 0,1, \ldots, C-1$}

    Predict $\hat{\bbs}^{c+1}$ with \eqref{eq:correction-step}
\EndFor

Set the corrected variable $\hat{\bbs}_{t+1}= \hat{\bbs}^{C}$
\EndFor
\end{algorithmic}
\caption{Online Time-Varying Graph Topology Inference}
\label{alg:complete}
\end{algorithm}

\smallskip\noindent Algorithm~\ref{alg:complete} shows the pseudocode for the general online TV-GTI framework.

\begin{remark}
 We point out that the framework can adopt different approximation schemes, such as extrapolation-based techniques, and can also include  time-varying constraint sets. The choice of approximation-scheme depends on the properties of the problem itself along with the required prediction accuracy. For an in-depth theoretical discussion regarding different prediction approaches and relative convergence results, refer to~\cite{bastianello2020primal}.
\end{remark}


\section{Network Models and Algorithms}
\label{sec:network-models}
In this section, we specialize the proposed framework to the three static topology inference models discussed in Section~\ref{sec:s-GTI}. Notice that the data dependency of data-driven graph learning algorithms is exerted via the empirical covariance matrix $\empcov$ of the graph signals; we have already shown this for the three considered models of Section~\ref{sec:s-GTI}. In other words, graph-dependent objective functions of the form $F(\bbS)$ could be explicitly expressed through their parametrized version $F(\bbS; \empcov)$. This rather intuitive, yet crucial observation, is central to render the proposed framework model-independent and adaptive, as explained next.

\noindent\textbf{Non-stationarity.}
Relying on the explicit dependence of function $F(\cdot)$ on $\empcov$ and envisioning non-stationary environments, we let the algorithm be adaptive by discarding past information. That is,  function $F(\bbS; t)$ in \eqref{eq: composite} can be written as $F(\bbS; \empcov_t)$, with $\empcov_t$ the empirical covariance matrix, up to time $t$, with past data gradually discarded.
This makes the framework adaptive and model-independent. The adaptive behavior can be shaped by, e.g., the  exponentially-weighted moving average (EWMA) of the covariance matrix:
\begin{align}
\label{empcov-rule}
    \empcov_t= \gamma\empcov_{t-1} + (1- \gamma)\bbx_t\bbx_t^\top \quad t=1, 2 \ldots
\end{align}
where the  forgetting factor $\gamma \in (0,1)$ downweighs (for $\gamma \rightarrow 0$) or upweighs (for $\gamma \rightarrow 1$) past data contributions. For stationary environments, an option is the infinite-memory matrix covariance update $\empcov_t= \frac{t-1}{t}\empcov_{t-1} + \frac{1}{t}\bbx_t \bbx_t^\top$.

\subsection{Time-Varying Gaussian Graphical Model}
The GGM problem \eqref{eq:ggm}, adapted to a time-varying setting  following template \eqref{eq:time-varying} leads to:
\begin{subequations}
\begin{align}
\label{eq:f-ggm}
    f(\bbS; t) &= -\log \operatorname{det}(\bbS) +\operatorname{tr}(\bbS \hat{\mathbf{\Sigma}}_{t})\\
     g(\bbS;t) &= \iota_{\mathbb{\ccalS}}(\bbS)
\end{align}
\end{subequations}
where $\ccalS= \mathbb{S}_{++}^N$. In this case $g(\cdot)$ encodes the constraint set of positive definite matrices and  the regularization parameter is $\lambda=1$.

Since $\bbS$ is symmetric, we use the half-vectorization $\bbs=\operatorname{vech}(\bbS) \in \mathbb{R}^{k}$ to reduce the number of independent variables from $N^2$ to $k=N(N+1)/2$. Then, the gradient and the Hessian of the function $f(\cdot)$ in the h-space are respectively: 
\vspace{-0.5mm}
\begin{subequations}
\begin{align}
\label{eq:gradient-ggm}
  \nabla_\bbs f(\bbs; t)&=  \bbD^{\top} \operatorname{vec} (\hat{\mathbf{\Sigma}}_t - \bbS^{-1}) \\ \vspace{1cm}
\label{eq:hessian-ggm}
\nabla_{\bbs \bbs} f(\bbs; t) &= \bbD^{\top} (\bbS \otimes \bbS )^{-1} \bbD.
\end{align}
\end{subequations}
Likewise, the discrete-time derivative of the gradient is given by the partial mixed-order derivative~\cite{simonetto2016class}:
\begin{align}
\label{eq: time-derivative-ggm}
\nabla_{t \bbs} f(\bbs ; t)=\bbD^{\top} \operatorname{vec} (\hat{\mathbf{\Sigma}}_{t} - \hat{\mathbf{\Sigma}}_{t-1}).
\end{align}
Note the Hessian term  \eqref{eq:hessian-ggm} is time-independent, while the time-derivative of the gradient \eqref{eq: time-derivative-ggm} is graph-independent.

Now, by defining $\hat{\bbs}_t:= \operatorname{vech(\hat{\bbS}_t)} \in \mathbb{R}^{k}$, we can particularize Algorithm~\ref{alg:complete} to:
\begin{itemize}
    \item \textbf{Prediction:} with $\hat{\bbs}^{0}$  initialized as $\hat{\bbs}^{0}= \hat{\bbs}_t$, the prediction update is :
        \begin{align}
        \label{eq:prediction-ggm}
            \hat{\bbs}^{p+1} &= \mathbb{P}_{\ccalS}[\hat{\bbs}^{p} - 2\alpha_t (\nabla_\bbs f(\hat{\bbs}_{t} ; t) + \nonumber \\ &+ \nabla_{\bbs \bbs} f(\hat{\bbs}_{t} ; t)\left(\hat{\bbs}^{p}-\hat{\bbs}_{t}\right)  + h \nabla_{t \bbs} f(\hat{\bbs}_{t} ; t)  ) ]
        \end{align}
        for $p=0,1, \ldots, P-1$, where $\alpha_t$ is a (time-varying) step size. Equation~\eqref{eq:prediction-ggm} entails a descent step along the approximate function $\hat{f}(\cdot; t+1)$  in~\eqref{eq:approximate-function}, followed by the projection onto the convex set $\ccalS$; see Appendix~\ref{sec:app-a} for the definition of $\mathbb{P}_{\ccalS}(\cdot)$. Then, the prediction $\hat{\bbs}_{t+1 \mid t}$ is set to $\hat{\bbs}_{t+1 \mid t}= \hat{\bbs}^{P}$.

        \item \textbf{Correction}:  by setting  $\hat{\bbs}^{0}\!=\! \hat{\bbs}_{t+1 \mid t}$, the correction update is:
            \begin{equation}
            \label{eq:correction-ggm}
                 \hat{\bbs}^{c+1}
            =\mathbb{P}_{\ccalS} \left[ \hat{\bbs}^{c} - \beta_t \nabla f(\hat{\bbs}^c; t+1)  \right]
            \end{equation}
            for $c=0,1, \ldots, C-1$, where $\beta_t$ is a (time-varying) step size. Equation~\eqref{eq:correction-ggm} entails a descent step along the true  function $f(\cdot; t+1)$,  followed by the projection onto the  set $\ccalS$.  The correction $\hat{\bbs}_{t+1}$ is finally set to  $\hat{\bbs}_{t+1}= \hat{\bbs}^{C}$.
    \end{itemize}
    
    \noindent The prediction step \eqref{eq:prediction-ggm} instantiates \eqref{eq:prediction-step} to $\hat{\ccalT}= \mathbb{P}_\ccalS \circ (I - \alpha_t \nabla_{\bbs} \hat{f} )(\cdot)$,  where $I(\cdot)$ is the identity function $I(\bbs)=\bbs$. Similarly, the correction step \eqref{eq:correction-ggm} instantiates \eqref{eq:correction-step}  to  $\ccalT= \mathbb{P}_\ccalS \circ (I - \beta_t \nabla_{\bbs}f )(\cdot)$. The overall computational complexity of one PC iteration is dominated by the matrix inversion and matrix multiplication, incurring a cost of $\ccalO(N^3)$. A correction-only algorithm would also incur a cost of $\ccalO(N^3)$ per iteration. See Appendix~\ref{app:complexity} for details.
\subsection{Time-Varying Structural Equation Model} 
The SEM problem \eqref{eq:sem-estimation}, adapted to a time-varying setting with sparsity-promoting regularizer, leads to [cf. \eqref{eq:time-varying}]:
\begin{subequations}
\begin{align}
\label{eq:f-sem}
          f(\bbS; t) &=  \frac{1}{2}   [\tr(\bbS^2\empcov_t)   -2 \tr(\bbS\empcov_t)  + \tr(\empcov_t)]\\
        \label{eq:g-sem}
     g(\bbS;t) &= \|\bbS\|_{1} + \iota_{\mathbb{\ccalS}}(\bbS)
\end{align}
\end{subequations}
where $\ccalS=\{ \bbS \in \mathbb{S}^N \vert  \operatorname{diag}(\bbS)=\bb0, S(i,j)=S(j,i), i \neq j\}$  is the set of hollow symmetric matrices, and $\|\bbS\|_{1}=\|\vec(\bbS)\|_1$. 
Since $\bbS$ is symmetric and hollow, we operate on the hh-space to make the problem unconstrained and reduce the number of independent variables from $N^2$ to $l=N(N-1)/2$, through its hollow half-vectorization form $\bbs=\vechh(\bbS) \in \mathbb{R}^l$. In the hh-space, equations \eqref{eq:f-sem} and \eqref{eq:g-sem} become:
\begin{subequations}
\begin{align}
\label{eq:f-sem-vec}
    f(\bbs; t) &= \frac{1}{2}\bbs^\top \bbQ_t \bbs -2 \bbs^\top \hat{\bbsigma}_t + \frac{1}{2}\hat{\sigma}_t \\
    \label{g-sem-vec}
    g(\bbs;t) &= 2\|\bbs\|_1
\end{align}
\end{subequations}
where $\bbQ_t:=\bbD_h^\top(\empcov_t \otimes \bbI)\bbD_h$ with $\otimes$ denoting the Kronecker product, $ \hat{\bbsigma}_t=\vechh(\empcov_t)$, and $\hat{\sigma}_t= \tr(\empcov_t)$. Since $\bbQ_t \succeq 0$,  \eqref{eq:f-sem-vec} is convex. 

To solve the time-varying SEM (TV-SEM) problem, we derive the gradient and the Hessian of  function $f(\cdot)$ in the hh-space as:
\begin{subequations}
    \begin{align}
    \label{eq:gradient-sem}
     \nabla_\bbs f(\bbs; t) &= \bbQ_t \bbs -2 \hat{\bbsigma}_t \\ \vspace{1cm}
    \label{eq:hessian-sem}
     \nabla_{\bbs \bbs} f(\bbs; t) &= \bbQ_t
    \end{align}
\end{subequations}
Notice here how the Hessian is time-varying and independent on  $\bbs$, differently from the GGM case. The time derivative of the gradient is given by the partial mixed-order derivative:
\begin{align}
\label{eq:time-gradient-sem}
    \nabla_{t\bbs}f(\bbs; t)= \frac{1}{h}[ (\bbQ_t - \bbQ_{t-1})\bbs -2 (\hat{\bbsigma}_t - \hat{\bbsigma}_{t-1})]
\end{align}

Now, by defining  $\hat{\bbs}_t:= \vechh(\hat{\bbS}_t) \in \mathbb{R}^l$, we can particularize Algorithm~\ref{alg:complete} to:
\begin{itemize}
    \item \textbf{Prediction:}  set $\hat{\bbs}^0=\hat{\bbs}_t$. Then, the  prediction is the proximal-gradient update:
    \begin{subequations}
    \label{eq: steps-sem}
 \begin{align}
 \label{eq:prediction-sem-1}
     \bbu^p &= \hat{\bbs}^p - \alpha_t [\nabla_\bbs f(\hat{\bbs}_{t} ; t) + \nonumber \\ &+ \nabla_{\bbs \bbs} f(\hat{\bbs}_{t} ; t)\left(\hat{\bbs}^{p}-\hat{\bbs}_{t}\right)  + h \nabla_{t \bbs} f(\hat{\bbs}_{t} ; t)]\\
     \label{eq:prediction-sem-2}
     \hat{\bbs}^{p+1}&= \operatorname{sign}(\bbu^p) \odot [|\bbu^p|- 2\alpha_t\lambda\boldsymbol{1} ]_+
 \end{align}
     \end{subequations}
for $p=0, \ldots, P$. Equation~\eqref{eq:prediction-sem-1} entails a descent step along the approximate function $\hat{f}(\cdot; t+1)$ in~\eqref{eq:approximate-function}, followed by  the non-negative soft-thresholding operator in~\eqref{eq:prediction-sem-2}, which sets to zero all the (negative) edge weights of the graph obtained after the gradient
descent in \eqref{eq:prediction-sem-1}. See Appendix~\ref{sec:app-a} for the formal definition of proximal operator, leading to \eqref{eq:prediction-sem-1} and \eqref{eq:prediction-sem-2}.
The final prediction $\hat{\bbs}_{t+1\vert t}$ is set to $\hat{\bbs}_{t+1\vert t}= \hat{\bbs}^P$.

\smallskip\item \textbf{Correction:} set $\hat{\bbs}^0=\hat{\bbs}_{t+1\vert t}$. Then, the correction is the proximal-gradient update:
\begin{subequations}
\label{eq: steps-sem-correction}
\begin{align}
\label{eq:correction-sem-1}
    \bbu^c&= \hat{\bbs}^c - \beta_t \nabla f(\hat{\bbs}^c; t+1)\\
    \label{eq:correction-sem-2}
     \hat{\bbs}^{c+1}&= \operatorname{sign}(\bbu^c) \odot [|\bbu^c|- 2\beta_t\lambda\boldsymbol{1} ]_+
\end{align}
    \end{subequations}
for $c=0, \ldots, C-1$. Equation~\eqref{eq:correction-sem-1} entails a descent step along the true function $f(\cdot; t+1)$,  followed by the non-negative soft-thresholding operator in~\eqref{eq:correction-sem-2}.  Finally, $\hat{\bbs}_{t+1}=\hat{\bbs}^C$. 

\end{itemize}

\noindent The prediction step  \eqref{eq: steps-sem} instantiates \eqref{eq:prediction-step} to $\hat{\ccalT}= \prox_{\lambda g, \alpha_t} \circ (I - \alpha_t \nabla_{\bbs}\hat{f}) (\cdot)$. Similarly, the correction step \eqref{eq: steps-sem-correction} instantiates  \eqref{eq:correction-step} to $\ccalT= \prox_{\lambda g, \beta_t} \circ \ (I - \beta_t \nabla_{\bbs}f )(\cdot)$. The overall computational complexity of one PC iteration is dominated by the computation of matrix $\bbQ_t$, incurring a cost of $\ccalO(N^3)$. A correction-only algorithm would also incur a cost of $\ccalO(N^3)$ per iteration. See Appendix~\ref{app:complexity} for details.

\subsection{Time-Varying Smoothness-based Model}
The SBM model \eqref{eq:smoothness-estimation} adapted to a time-varying setting is:
\begin{subequations}
\begin{align}
\label{eq:f-smm}
        f(\bbS; t) &= \tr(\Diag(\bbS \boldsymbol{1})\empcov_t ) - \tr(\bbS\empcov_t)\\
        \label{eq:g-smm}
     g(\bbS;t) &=  \frac{\lambda_1}{4}\|\bbS\|_F^2 - \lambda_2 \boldsymbol{1}^\top \log (\bbS \boldsymbol{1}) + \iota_{\mathbb{\ccalS}}(\bbS)
\end{align}
    \end{subequations}
where  $\ccalS=\{\bbS \in \mathbb{S}^N \vert {\rm diag}({\bf S}) = {\bf 0},  S(i,j) \!=\!S(j,i) \geq 0, i \neq j\}$ is the set of hollow symmetric matrices. The log barrier term $\log (\bbS \boldsymbol{1})$ is applied entry-wise and forces the nodes degree vector $\bbd=\bbS \boldsymbol{1}$ to be positive while avoiding the trivial solution. The Frobenius norm term $\|\bbS\|_F^2$ controls the sparsity of the graph.

By operating in the hh-space, equations~\eqref{eq:f-smm} and \eqref{eq:g-smm} become\footnote{We move the log-barrier  and Frobenius norm terms of $g(\cdot)$ function \eqref{eq:g-smm} into the $f(\cdot)$ function to fit the structure of the general template.}:
\begin{subequations}
\begin{align}
\label{eq:f-smm-vec}
    f(\bbs;t) & \!= \bbs^\top(\bbK^\top \hat{\bbsigma}_d\!- \! 2\hat{\bbsigma}_t)\! -\! \lambda_2\boldsymbol{1}^\top \log (\bbK\bbs)\! +\! \frac{\lambda_1}{2} \|\bbs\|^2\\
    \label{eq:g-smm-vec}
    g(\bbs;t) & = \iota_{\mathbb{R}_+}(\bbs)
\end{align}
\end{subequations}
where $\bbK \in \{0,1\}^{N \times l}$ is the binary matrix such that $\bbd=\bbS\boldsymbol{1}= \bbK \bbs$, $\hat{\bbsigma}_d= \diag(\empcov_t)$ and $\hat{\bbsigma}_t= \vechh(\empcov_t)$. 

To apply the proposed framework to solve the time-varying SBM (TV-SBM) problem, we derive the gradient and the Hessian of  function $f(\cdot)$ in the hh-space as follows:
\begin{subequations}
\begin{align}
\label{eq:gradient-smm}
 \nabla_\bbs f(\bbs; t) &= \lambda_1\bbs - \lambda_2 \bbK^\top (\boldsymbol{1} \oslash \bbK\bbs) + \bbz_t  \\ \vspace{1cm}
\label{eq:hessian-smm}
 \nabla_{\bbs \bbs} f(\bbs; t) &= \lambda_1\bbI + \lambda_2 \bbK^\top \Diag(\boldsymbol{1}\oslash (\bbK\bbs)^{\circ 2})\bbK
\end{align}
\end{subequations}
where $\oslash$ and $^\circ$ represent the Hadamard division and power, respectively. The time derivative of the gradient is given by the partial mixed-order derivative:
\begin{align}
\label{eq:time-grad-sbm}
    \nabla_{t\bbs}f(\bbs; t)= \frac{1}{h} (\bbz_t - \bbz_{t-1})
\end{align}
where $\bbz_t= \bbK^\top \hat{\bbsigma}_d\!- \! 2\hat{\bbsigma}_t$.
Now, by defining $\hat{\bbs}_t:= \vechh(\hat{\bbS}_t) \in \mathbb{R}^l$, we can particularize Algorithm~\ref{alg:complete} to:
\begin{itemize}
    \item \textbf{Prediction:} with $\hat{\bbs}^{0}$  initialized as $\hat{\bbs}^{0}= \hat{\bbs}_t$, the prediction update is:
        \begin{align}
        \label{eq:prediction-smm}
            \hat{\bbs}^{p+1}
            &=\mathbb{P}_{\bbs \succeq \bb0}[\hat{\bbs}^{p} - 2\alpha_t (\nabla_\bbs f(\hat{\bbs}_{t} ; t) + \nonumber \\ &+ \nabla_{\bbs \bbs} f(\hat{\bbs}_{t} ; t)\left(\hat{\bbs}^{p}-\hat{\bbs}_{t}\right)  + h \nabla_{t \bbs} f(\hat{\bbs}_{t} ; t)  ) ]
        \end{align}
        for $p=0,1, \ldots, P-1$. Equation~\eqref{eq:prediction-smm} entails a descent step along the approximate function $\hat{f}(\cdot; t+1)$  in~\eqref{eq:approximate-function}, followed by the projection onto the non-negative orthant. Then, the prediction $\hat{\bbs}_{t+1 \mid t}$ is set to $\hat{\bbs}_{t+1 \mid t}= \hat{\bbs}^{P}$.

        \item \textbf{Correction}:  by setting  $\hat{\bbs}^{0}= \hat{\bbs}_{t+1 \mid t}$, the correction update is:
            \begin{equation}
            \label{eq:correction-smm}
                 \hat{\bbs}^{c+1}
            =\mathbb{P}_{\bbs \succeq \bb0} \left[ \hat{\bbs}^{c} - \beta_t \nabla f(\hat{\bbs}^c; t+1)  \right],
            \end{equation}
            for $c=0,1, \ldots, C-1$. Equation~\eqref{eq:correction-smm} entails a descent step along the true  function $f(\cdot; t+1)$,  followed by the projection onto the non-negative orthant.
            Finally,  $\hat{\bbs}_{t+1}= \hat{\bbs}^{C}$.
    \end{itemize}
    
    \noindent The prediction step \eqref{eq:prediction-smm}  instantiates \eqref{eq:prediction-step} to $\hat{\ccalT}= \mathbb{P}_{\bbs \succeq \bb0} \circ (I - \alpha_t \nabla_{\bbs} \hat{f} )(\cdot)$. Similarly, the correction step \eqref{eq:correction-smm} instantiates \eqref{eq:correction-step} to  $\ccalT= \mathbb{P}_{\bbs \succeq \bb0} \circ (I - \beta_t \nabla_{\bbs}f )(\cdot)$. The overall computational complexity per iteration is dominated by the computation  of the gradient $\nabla_\bbs f(\bbs; t)$ (or the Hessian if $P>1$), incurring a cost of $\ccalO(N^2)$ (or $\ccalO(N^3)$ if $P>1$). See Appendix~\ref{app:complexity} for details.

\section{Convergence Analysis}
\label{sec:convergence-analysis}
In this section, we first discuss the convergence of Algorithm~\ref{alg:complete} and the associated error bounds. As solver we consider the proximal gradient  $\hat{\ccalT}=\ccalT=\prox_{g,\rho} \circ (I - \rho \nabla_\bbs f)(\cdot)$~\cite{ryu2016primer, combettes2011proximal}.  Then, we show how the parameters of the three introduced models are involved in the bounds. To ease notation, we use $\bbs \in \mathbb{R}^{p}$ to indicate the  vectorization of matrix variable $\bbS \in \mathbb{R}^{N \times N}$ [cf. Section~\ref{subsec:reduction}].

For this analysis, we need the following mild assumptions.

\begin{assumption} \label{ass1}
The function $ f:\mathbb{R}^{p} \times {\mathbb N}_+ \rightarrow {\mathbb R}$ is $m$-strongly convex and $L$-smooth  uniformly in $t$, i.e., ${m \bbI \preceq \nabla_{\bbs \bbs} f(\bbs;t) \preceq L \bbI, \;\; \forall \; \bbs, t}$,
%
while the function $g:\mathbb{R}^{p} \times {\mathbb N}_+ \rightarrow \mathbb{R} \cup \{+ \infty\}$ is closed convex and proper, or $g(\cdot; t)=0 $, for all $t \in \mathbb{N_+}$.
\end{assumption}
This guarantees that problem~\eqref{eq:time-varying} admits a unique solution for each time instant, which in turn guarantees uniqueness of the solution trajectory $\{\bbs_t^\star\}_{t=1}^\infty$.

\begin{assumption}\label{ass2}
The gradient of function $f(\cdot)$ has bounded time derivative, i.e. $\exists \; C_0 >0 $ such that $\|\nabla_{t\bbs}f(\bbs;t)\| \leq C_0 \;  \forall \; \bbs \in \mathbb{R}^p, t \in \mathbb{N}_+$. 
\end{assumption}
This guarantees that the solution trajectory is Lipschitz in time.

\begin{assumption}\label{ass3}
The predicted function $\hat{f}(\cdot; t+1)$ is $m$-strongly convex and $L$-smooth uniformly in $t$; and $\hat{g}(\cdot; t+1)$ is closed, convex and proper. 
\end{assumption}
This implies that the prediction problem \eqref{eq:prediction} belongs to the same class as the original problem, i.e., the functions of the two problems share the same strong convexity and Lipschitz constants $m$ and $L$.  Therefore, the same solver can be applied for the prediction and correction steps, i.e., $\hat{\ccalT}= \ccalT$.

\begin{assumption}
\label{ass:graph}
The matrix $\bbS$ of \eqref{eq: composite} has finite entries, i.e.,  $- \infty \!<S(i,j)\!<\! + \infty$, for all $i,j$.
\end{assumption}

 This guarantees $\|\bbS\|  < + \infty$, i.e., $\bbS$ is a bounded operator, and it holds in practical scenarios. In particular, it is known that (finite) weighted graphs exhibit bounded eigenvalues, see \cite{zhan2005extremal}\cite{das2008sharp}. Notably, if $\bbS$ is a normalized Laplacian, then $\|\bbS\|\!=\! 2$.

\smallskip\noindent Similarly, assumptions~\ref{ass1}-\ref{ass3} are mild and hold for the considered models, as we show next. 

\begin{prop}
The three considered models of Section~\ref{sec:network-models} can be $m$-strongly convex and $L$-smooth uniformly in $t$, for some scalar $m$ and $L$, as supported by the following claims.
 
 \begin{claim}
Denote with $\xi>0$ and $0 < \chi < \infty$ the minimum and maximum admissible eigenvalues of the precision matrix $\bbS$, respectively;  i.e., consider the set $\ccalS=\{\bbS \in \mathbb{S}_{++}^N \vert \xi \bbI \! \preceq \! \bbS \! \preceq \! \chi \bbI \}$. Then, for the TV-GGM function $f(\cdot;t)$ in \eqref{eq:f-ggm}, it holds:
 \begin{align}
    m = 1/\chi \quad \quad L= 2/ \xi.
\end{align}
 \end{claim}

 \begin{claim}
Denote with $\lambda_{\text{min}}$ and $\lambda_{\text{max}}$ the smallest and highest eigenvalues for the set of empirical covariance matrices obtained with graph signals obeying~\eqref{eq:sem-data-model}. Then, for the TV-SEM function $f(\cdot;t)$ in \eqref{eq:f-sem-vec}, it holds:
\begin{align}
    m =\lambda_\text{min} \quad \quad L = 2 \lambda_{\text{max}}.
\end{align}
\end{claim}

\begin{claim}
Consider the TV-SBM function $f(\cdot;t)$ in \eqref{eq:f-smm-vec}, and recall that the log-barrier term avoids isolated vertices, i.e., $\bbd \succ \bb0$. Denote with $d_\text{min} > 0$ the minimum degree of the GSO search space. Under these assumptions, it holds:
\begin{align}
    m = 2 \lambda_1 \quad \quad L = 2 \lambda_2 (N-1) d_\text{min}^{-2}.
\end{align}
See Appendix~\ref{proof:bounds} for a proof of Claim~1-3.
\end{claim}
\end{prop}

Thus, Assumption~\ref{ass1} holds since the Hessian of $f(\cdot;t)$ is bounded over time and $g(\cdot; t)$ is closed, convex and proper by problem construction; Assumption~\ref{ass2}  holds since $\nabla_{t\bbs}f(\bbs;t)$ is the difference between bounded vectors which involve covariance matrices not too different from each other (one is the rank-one update of the other), which is finite as long as the graph signals are bounded, see~\eqref{empcov-rule} and, e.g., ~\eqref{eq:time-grad-sbm}. Assumption~\ref{ass3} holds since $\hat{f}(\cdot; t+1)$ is a quadratic approximation of $f(\cdot;t)$ [cf.~\eqref{eq:approximate-function}] and $\hat{g}(\cdot; t+1)= g(\cdot;t)$, thus inheriting the properties of $f(\cdot; t)$ and $g(\cdot; t)$, which satisfy Assumption~\ref{ass1}.

With this in place, we are now ready to show two different error bounds incurred during the prediction and correction steps performed by Algorithm~\ref{alg:complete}, describing its sub-optimality as function of the model and algorithm's parameters. First, we show the error bound between the optimal prediction solution  $\bbs_{t+1 \vert t}^\star$ and the associated optimal correction $\bbs_{t+1}^\star$, which solve problems \eqref{eq:prediction} and  \eqref{eq:correction}, respectively.

\begin{prop}
 Let Assumptions~\ref{ass1}-\ref{ass3} hold. Consider also the Taylor expansion based prediction~\eqref{eq:approximate-function} for $f(\cdot; t)$ and the one-step back prediction for $g(\cdot; t)$. Then, the distance between the optimal prediction solution  $\bbs_{t+1 \vert t}^\star$, solving problem \eqref{eq:prediction}, and the associated optimal correction $\bbs_{t+1}^\star$, solving problem \eqref{eq:correction}, is upper bounded by:
\begin{align}
\label{eq:instantaneous-error}
    \| \bbs_{t+1 \vert t}^\star - \bbs_{t+1}^\star \| \leq \frac{2L}{m} \|\hat{\bbs}_t - \bbs_t^\star\| + \frac{2 C_0 h}{m}(1 + \frac{L}{m})
\end{align}
where $\hat{\bbs}_t$ is the approximate solution of the correction problem \eqref{eq:correction} at time $t$.
\end{prop}

\begin{proof}
Follows from \cite[Lemma 4.2]{bastianello2020primal} in which  constant $D_0=0$  by considering a static function $g(\cdot)$.
\end{proof}

This bound enables us to measure how far the prediction is from the true corrected topology at time $t+1$. It depends on the estimation error $\hat{\bbs}_t - \bbs_t^\star$ achieved at time $t$, the ratio $L/m$ and the variability of the function gradient $\nabla_{t\bbs}f(\bbs;t)$. The bound suggests that a small gap can be achieved if \textit{i)} the ratio $L/m$ is small, which for the three considered models translates in having a small condition number for the involved covariance matrices or GSOs; and \textit{ii)} the time-gradient $\nabla_{t\bbs}f(\bbs;t)$ at consecutive time steps does not change significantly, which holds when the considered models have similar covariance matrices at adjacent time instants, i.e., the data statistics do not change too rapidly (see e.g.~\eqref{eq: time-derivative-ggm} and \eqref{eq:time-gradient-sem}).

\smallskip\noindent  

Finally, we bound the error sequence $\{ \|\hat{\bbs}_t - \bbs_t^\star\|_2, t=1, 2, \ldots\}$ achieved by Algorithm~\ref{alg:complete} by means of the following non-asymptotic performance guarantee, which is an adaptation of \cite[Proposition 5.1]{bastianello2020primal}.
\begin{theorem}
\label{theorem:error}
Let Assumptions~\ref{ass1} and \ref{ass3} hold, and consider two scalars $\{d_t, \phi_t\} \in \reals_{+}$ such that:
\begin{align}
    \|\bbs_{t+1}^\star - \bbs_t^\star\|  \leq d_t \quad \text{and} \quad 
    \| \bbs_{t+1 \vert t}^\star - \bbs_{t+1}^\star \| & \leq \phi_t
\end{align}
for any $t \in \mathbb{N}_+$. Let also the prediction and correction steps use the same step-sizes $\rho_t= \alpha_t= \beta_t$. Then, by employing $P$ prediction and $C$ correction steps with the proximal gradient operator $\ccalT=\prox_{g,\rho_t} \circ (I - \rho_t \nabla_\bbs f)(\cdot)$, the sequence of iterates $\{\hat{\bbs}_{t}\}$ generated by Algorithm~\ref{alg:complete} satisfies:
\begin{align}
    \label{eq:error-bound}
    \|\hat{\bbs}_{t+1} \!-\! \bbs_{t+1}^\star\|_2 \!\leq\! q_t^C (q_t^P \|\hat{\bbs}_{t} - \bbs_{t}^\star\| \!+\! q_t^P d_t \!+\! (1+ q_t^P) \phi_t)
\end{align}
where $q_t= \max\{|1 - \rho_t m_t|, |1- \rho_t L_t| \} \in (0,1)$ is the contraction coefficient \cite{beck2017first}.  
\end{theorem}

\begin{proof}
Follows from \cite[Proposition 5.1]{bastianello2020primal} and \cite[Lemma 2.5]{bastianello2020primal}, with variables $\lambda= q_t$ and $\chi=\beta=1$. 
\end{proof}
Theorem~\eqref{eq:error-bound} states that the sequence of estimated graphs $\{\bbs_t\}_{t \in \mathbb{N}_+}$ hovers around the optimal trajectory $\{\bbs_t^\star\}_{t \in \mathbb{N}_+}$ with a distance depending on: \textit{i)} the numbers $P$ and $C$ of iterations; \textit{ii)} the estimation error achieved at the previous time instant $\|\hat{\bbs}_{t} - \bbs_{t}^\star\|$; and \textit{iii)} the quantities $d_t$ and $\phi_t$.
Moreover, \eqref{eq:error-bound} is a contraction (i.e., $q_t^{C+P}<1$) when $\rho_t < 2/L_t$; in this case the initial starting point $\hat{\bbs}_0$ does not influence the error $\hat{\bbs}_{t+1} \!-\! \bbs_{t+1}^\star$ asymptotically, since the first term in~\eqref{eq:error-bound} vanishes. However, the terms $d_t$ and $\phi_t$ keep impacting the error also asymptotically, as long as the problem is time-varying;  if the problem becomes static, i.e., the solution stops varying, then $d_t=\phi_t=0$, and the overall error asymptotically goes to zero.

\section{Numerical Results}
\label{sec:numerical-result}

In this section, we show with numerical results how Algorithm~\ref{alg:complete}, specialized to the three models (TV-GGM, TV-SEM, TV-SBM), can track the offline solution~\eqref{eq:time-varying} obtained by the respective instantiations.  For all the experiments, we initialize the empirical covariance matrix $\empcov_0$  with some samples acquired prior to the analysis. We consider $P=1$ prediction steps and $C=1$ correction steps, which is the challenging setting of having the minimum iteration budget  for streaming scenarios. We measure the convergence of Algorithm~\ref{alg:complete} via the normalized squared error (NSE) between the algorithm's estimate $\hat{\bbs}_t$ and the optimal (offline) solution $\bbs_t^\star$:
\begin{align}
    \text{NSE}(\hat{\bbs}_t, \bbs_t^\star) = \frac{\|\hat{\bbs}_{t} - \bbs_{t}^\star\|_2^2}{\|\bbs_{t}^\star\|_2^2}.
\end{align}
We use CVX \cite{grant2008cvx} as solver for the offline computations, and report the required computational time in seconds achieved by Algorithm~\ref{alg:complete} and CVX.


\begin{figure*}%
\centering
\begin{subfigure}{0.33\textwidth}
\centering
\captionsetup{justification=centering}
\vspace{.1cm}
\includegraphics[width=.87\textwidth, trim= 0.6cm 0.1cm 1cm 0.75cm, clip=true]{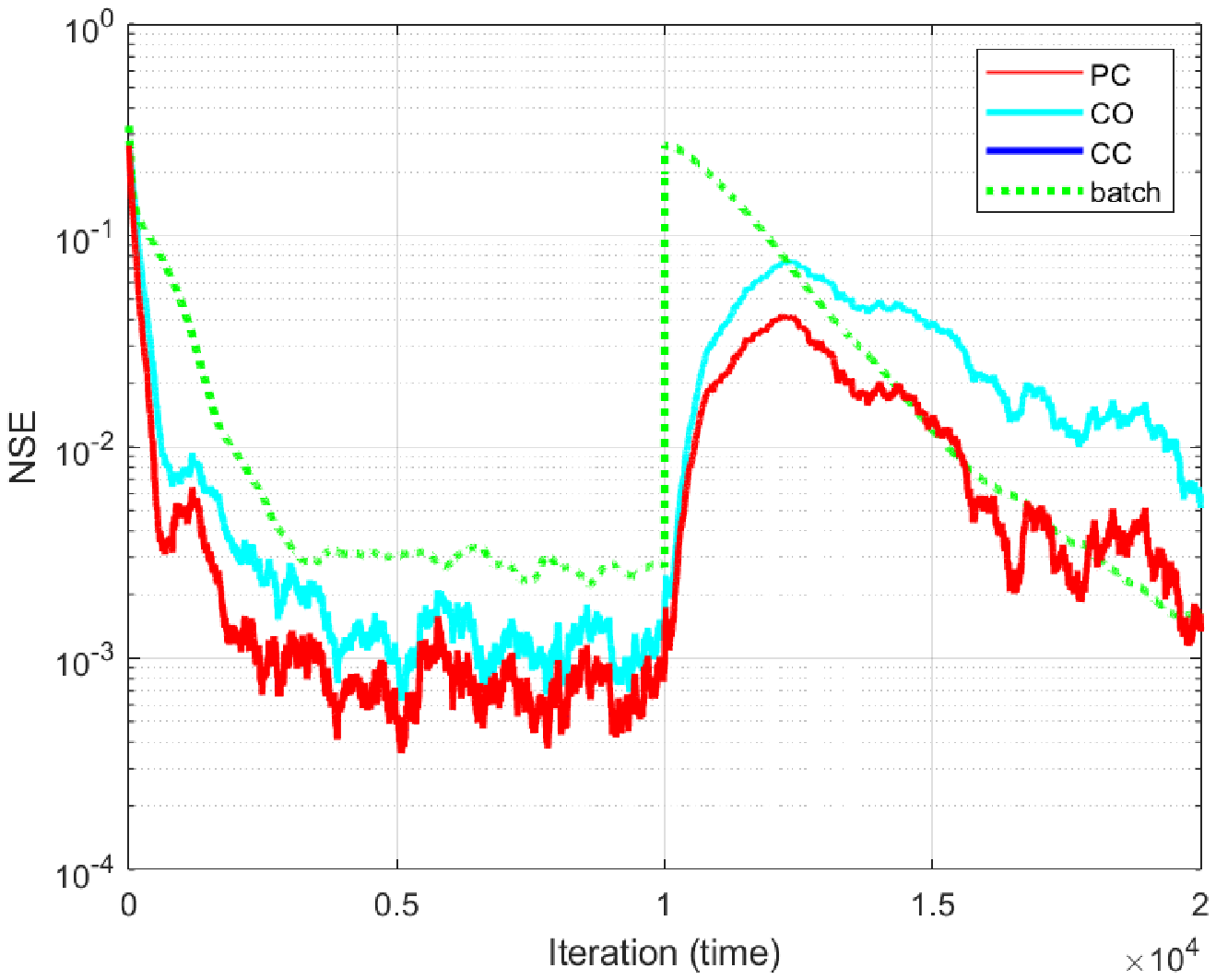}%
\vspace{-.2cm}
\caption{\textbf{TV-GGM.} $N\!=\!18$, $\alpha\!=\beta\!=\!10^{\!-\!2}\!$, $\gamma\!=\!99.9\!\times\!10^{\!-\!2}$ \newline}%
\label{ltop}%
\end{subfigure}%
\begin{subfigure}{0.33\textwidth}
\centering
\captionsetup{justification=centering}
\includegraphics[width=.95\textwidth, trim= 0 0 0 5mm, clip=true]{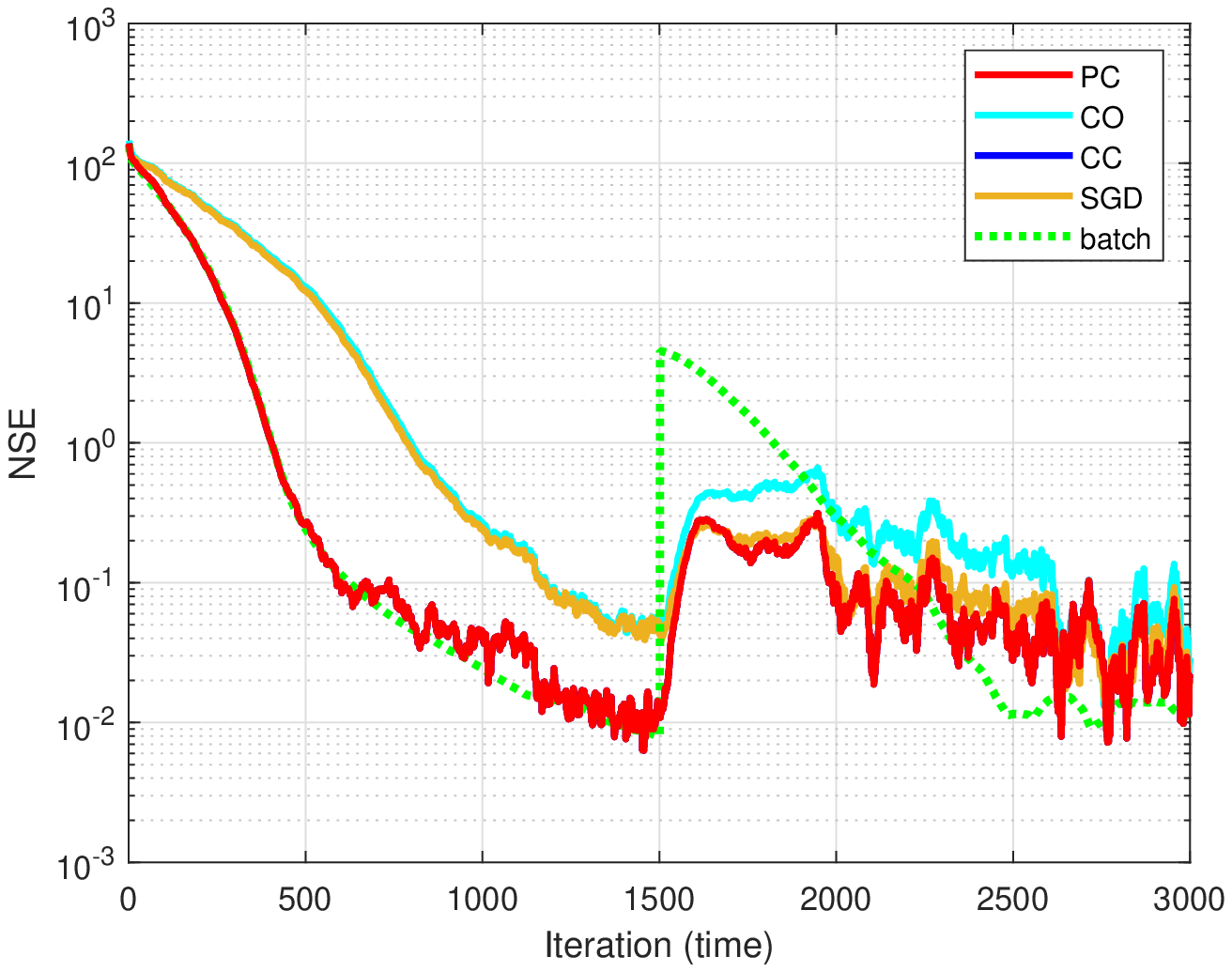}
\vspace{-.2cm}
\caption{\textbf{TV-SEM.} $N\!=\!28$, $\alpha\!=\beta\!=\!0.1\!\times\!10^{\!-\!2} $, \newline  $\lambda\!=\!0.5$, $\gamma\!=\!99\!\times\!10^{\!-\!2}$}%
\label{ctop}%
\end{subfigure}%
\begin{subfigure}{0.33\textwidth}
\centering
\captionsetup{justification=centering}
\includegraphics[width=.95\textwidth, trim= 0 0 0 5mm, clip=true]{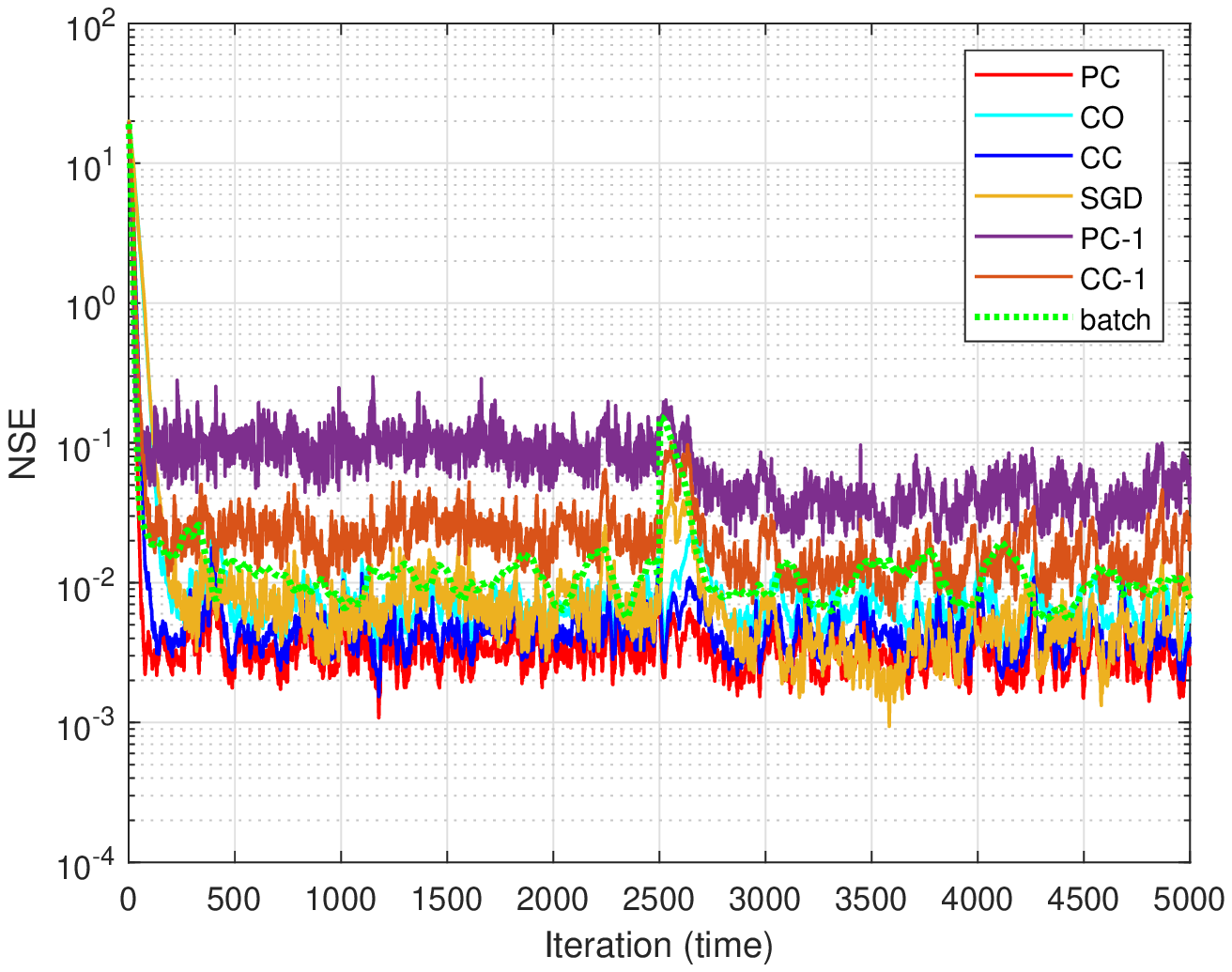}%
\vspace{-.2cm}
\caption{\textbf{TV-SBM.} $N\!=\!28$, $\alpha\!=\beta\!=\!0.1\!\times\!10^{\!-\!2}$, \newline $\lambda_1\!=\!10$, $\lambda_2\!=\!10$, $\gamma\!=\!99\!\times\!10^{\!-\!2}$}%
\label{rtop}%
\end{subfigure}%
\\ 
\begin{subfigure}{0.33\textwidth}
\centering
\captionsetup{justification=centering}
\includegraphics[width=.95\textwidth, trim= 0 0 0 5mm, clip=true]{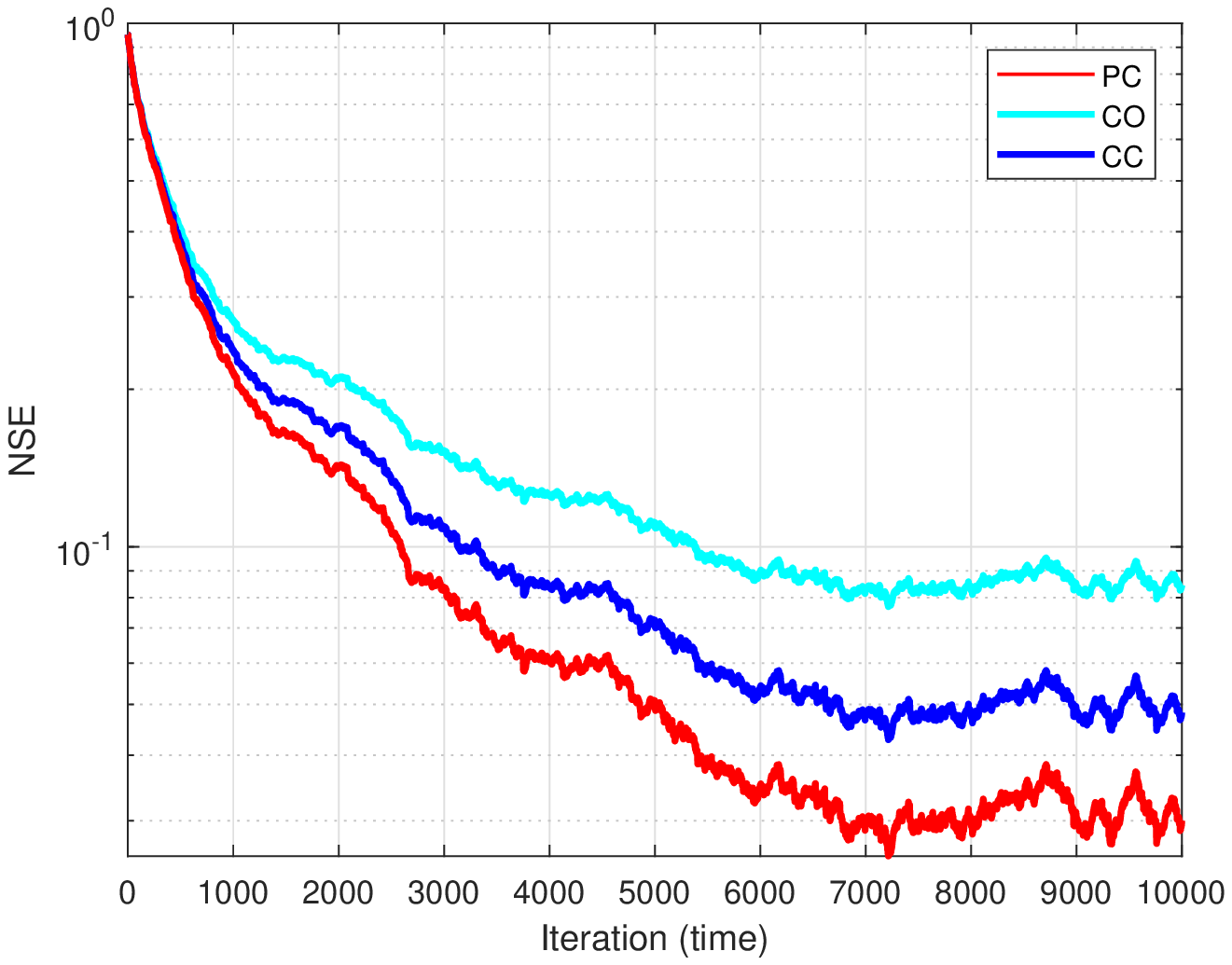}%
\vspace{-.2cm}
\caption{\textbf{TV-GGM.} $N\!=\!18$, $\alpha\!=\beta\!=\!0.1\!\times\!10^{\!-\!2} $, \newline $\gamma\!=\!99.9\!\times\!10^{\!-\!2}$ }%
\label{lbot}%
\end{subfigure}%
\begin{subfigure}{0.33\textwidth}
\centering
\captionsetup{justification=centering}
\includegraphics[width=.95\textwidth, trim= 0 0 0 5mm, clip=true]{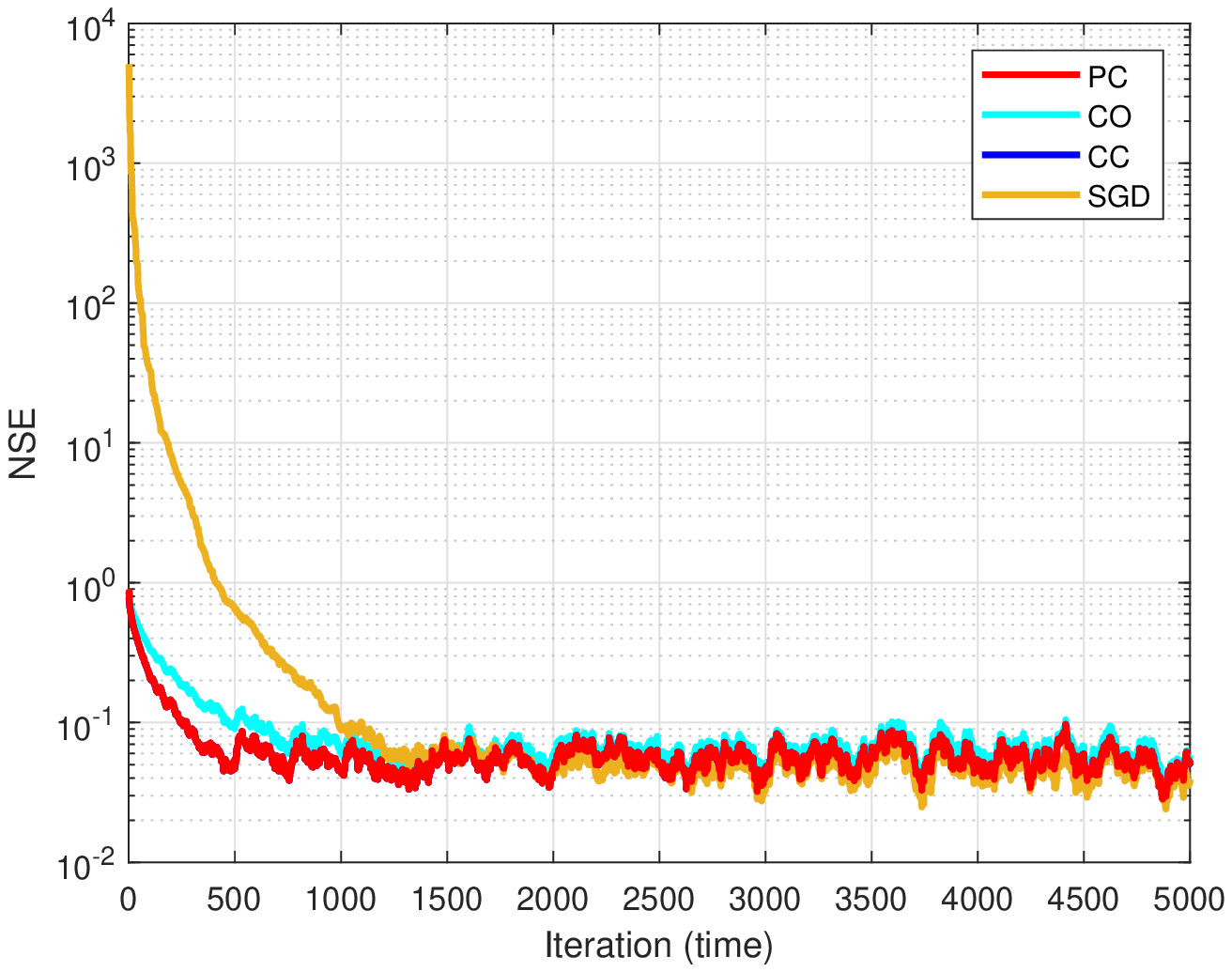}%
\vspace{-.2cm}
\caption{\textbf{TV-SEM.} $N\!=\!28$,  $\alpha\!=\beta\!=\!0.5\!\times\!10^{\!-\!2} $,  $\lambda\!=\!5e\!-\!2$, $\gamma\!=\!99\!\times\!10^{\!-\!2}$}%
\label{cbot}%
\end{subfigure}%
\begin{subfigure}{0.33\textwidth}
\centering
\captionsetup{justification=centering}
\includegraphics[width=.95\textwidth, trim= 0 0 0 5mm, clip=true]{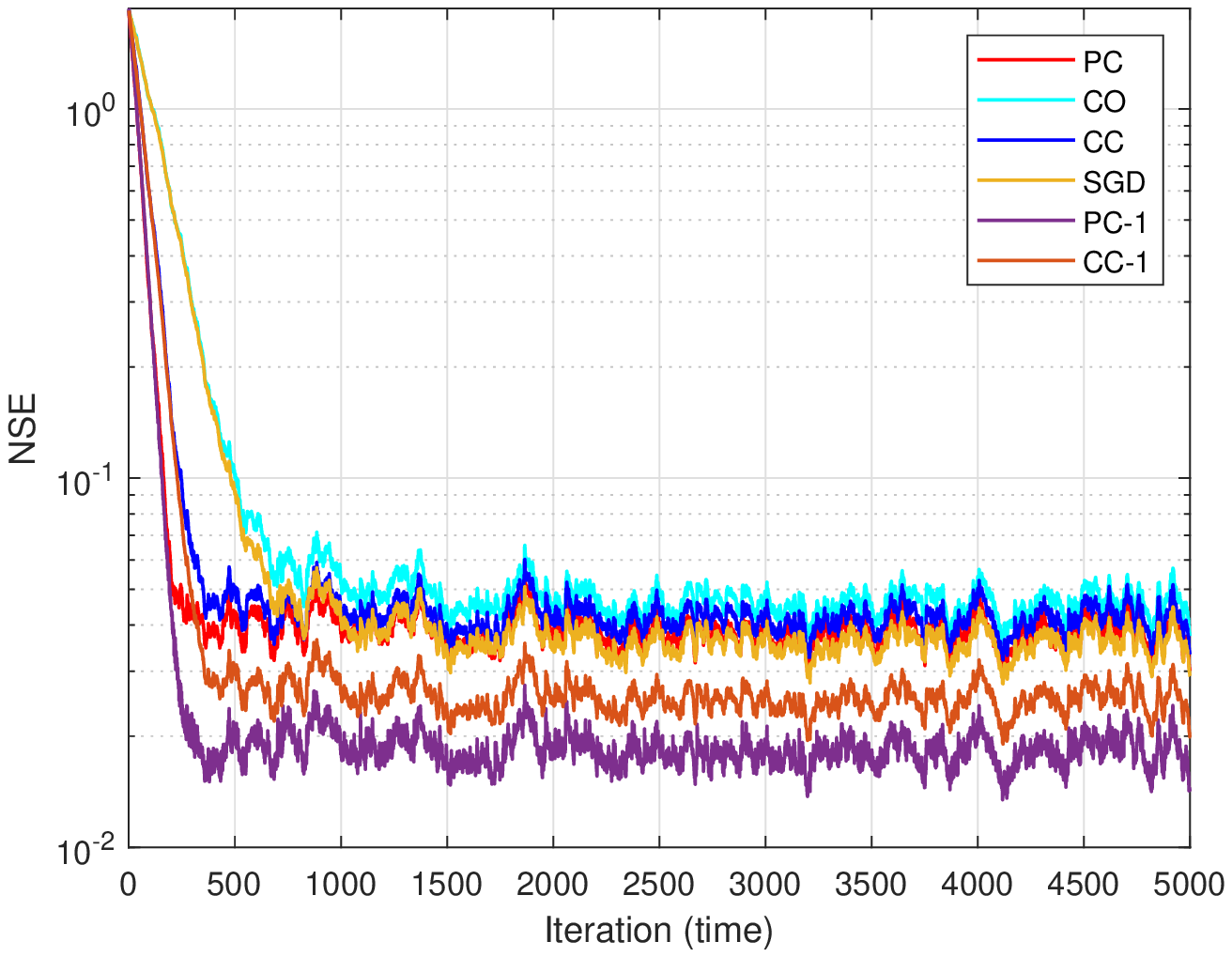}%
\vspace{-.2cm}
\caption{\textbf{TV-SBM.} $N=28$, $\alpha\!=\beta\!= \!0.1\!\times\!10^{\!-\!2} $,  $\lambda_1=1$,  $\lambda_2=10$, $\gamma=\!99\!\times\!10^{\!-\!2}$}%
\label{rbot}%
\end{subfigure}%
\vspace{-.2cm}
\caption{Normalized squared error (NSE) for the piecewise-constant (top row) and smooth (bottom row) synthetic scenarios between our online solution $\hat{\bbs}_t$ (or the other variants reported in the legend) with respect to the offline solution $\bbs_t^\star$ obtained with CVX. For the piecewise-constant scenario, it is also illustrated the NSE between the PC solution and the batch solution (green curve). Stochastic implementations are available for a subset of methods due to numerical instabilities caused by the rank-one matrix operations involved.}
\label{fig:NMSE}\vspace{-4mm}
\end{figure*}
%

\subsection{Synthetic Data}
\label{sec:synthetic}
We generate a synthetic  (seed) random graph $\bbS_0$ of $N$ nodes using the GSP toolbox~\cite{perraudin2014gspbox}. Then, edges abide two different temporal evolution patterns: \textit{i)} \textit{piecewise constant}; and  \textit{ii) }\textit{smooth} temporal variation. Finally, we generate the stream of data according to the three considered models [cf. Section~\ref{sec:network-models}] for $T$ time instants. 

\smallskip\noindent \textbf{Piecewise.} For the piecewise constant scenario, we randomly select $\lceil N/2 \rceil$ nodes of the initial graph $\bbS_0$ and double the weight of their edges, after $T/2$ samples.
Then, for $t= \{1, \ldots, T\}$ we generate each graph signal $\bbx_t$ according to the three models: $1)$ for the TV-GGM, we use $\bbx_t \sim \ccalN(\mathbf{0},\mathbf{\Sigma}_t)$, where $\mathbf{\Sigma}_t= \bbS_t^{-1}$; $2)$ for the TV-SEM we use  $\bbx_t= (\bbI - \bbS_t)^{-1}\bbe_t$ [cf.~\eqref{eq:sem-data-model}], with noise variance $\sigma_e^2= 0.5$; and $3)$ for the TV-SBM we use  $\bbx_t \sim  \ccalN(\mathbf{0}, \bbL_t^\dagger +  \sigma_e^2 \bbI_N)$ as in~\cite{dong2016learning} with $\sigma_e^2= 0.5$.

\smallskip\noindent \textbf{Smooth.} For the smooth scenario, starting from the initial graph $\bbS_0$, the evolution pattern follows an edge-dependent behavior, $S_t(i,j)= S_0(i,j) ( 1 + e^{-0.01ijt})$ for $t= \{1, \ldots, T\}$. This means that each edge follows an exponential decaying behavior, with the decaying factor depending on the edge itself. The data are generated as in the piecewise constant scenario.

\noindent For the results, we will compare  the following methods:

\begin{itemize}
    \item \textbf{Prediction-correction (PC)} \emph{red curve}: this is the proposed Algorithm~\ref{alg:complete}  specialized to one of the three models, with $P=C=1$.
    
    \item \textbf{Correction-only (CO)} \emph{cyan curve}: this is a prediction-free algorithm which only considers  the original problem~\eqref{eq:correction} and applies $C=1$ iteration of the recursion~\eqref{eq:correction-step}. It is equivalent to Algorithm~\ref{alg:complete} with $P=0, C=1$. We consider this algorithm to study the benefits of the prediction step performed by PC. 
    
    \item \textbf{Correction-correction (CC)} \emph{blue curve}: this is a prediction-free algorithm which only  considers  the original problem~\eqref{eq:correction} and applies $C=2$ iterations of the recursion~\eqref{eq:correction-step}. It is equivalent to Algorithm~\ref{alg:complete} with $P=0, C=2$. This is a more fair comparison than CO, since the number of iterations is the same as the one of PC. 
    
    \item  \textbf{Stochastic gradient descent (SGD)} \emph{ochre curve}: this is a prediction-free and memory-less version of the algorithm which only considers the last acquired graph signal. That is, the empirical covariance matrix $\empcov_t=\bbx_t\bbx_t^\top$   in~\eqref{empcov-rule} is just a rank-one update, achieved by setting $\gamma=0$. We consider this to show how much the temporal variability of the function, captured by the time-derivative of the gradient in PC, affects the algorithm's convergence.

    
    \item \textbf{Prediction-correction rank-one (PC-1)} \emph{purple curve}: this is a rank-one (stochastic) implementation of the PC algorithm; i.e.,  $\empcov_t= \bbx_t\bbx_t^\top$ for the update in~\eqref{empcov-rule}, and $P=C=1$. Notice that, differently from SGD, it also uses the time-derivative of the gradient, which in this case is the difference between two rank-one covariance matrices (thus the length of the memory is equal to one). We consider this algorithm to check the impact of the prediction step in a stochastic implementation of PC;
    
    \item \textbf{Correction-correction rank-one (CC-1)} \emph{orange curve}: this is a rank-one (stochastic) implementation of the CC algorithm; i.e., it considers $\empcov_t= \bbx_t\bbx_t^\top$ for the update in~\eqref{empcov-rule}, and $P=0, C=2$. It can be seen as a two-step SGD, and we consider it to study whether the prediction step of PC-1 is beneficial for stochastic implementations.

\end{itemize}

\noindent In addition, for the piecewise constant scenario, we also report (green curve) the NSE between the PC solution and the batch solution obtained having all the relevant data in advance, i.e., the solution that would be obtained with a static graph learning algorithm on the intervals where the graph remains constant. In general, a fair comparison  can be made within the rank-one implementations (SGD, PC-1 and CC-1)  and within the memory-aware ones (PC, CO, CC).

\smallskip\noindent \textbf{Results.} 
The NSE achieved by Algorithm~\ref{alg:complete} for the three models is shown in Fig.~\ref{fig:NMSE}, for both the piecewise constant (top row) and smooth (bottom row) scenarios.  We use fixed step sizes for all the experiments.  Notice that the only effect of the functions' hyperparameters is to shape the batch solution $\bbs_t^\star$ (and hence the time-varying trajectory $\hat{\bbs}_t$ at convergence).  Thus, we run Algorithm~\ref{alg:complete} with different hyperparameters\footnote{The search space intervals for the hyperparameters are the following: $\alpha, \beta \in (0.01, 1) \times 10^{-2}$, $\lambda \in (0.005, 5)$, $\lambda_1, \lambda_2 \in (1,10)$ ,  $ \gamma \in \{97,99, 99.9\} \times 10^{-2}$.} and manually select them by ensuring that the trivial and complete graphs are excluded; the selected ones are displayed, together with the other algorithm's parameters, in the captions of Fig.~\ref{fig:NMSE}.

\smallskip\noindent\textit{GGM.} Fig.~\ref{ltop} and Fig.~\ref{lbot} show the results for the piecewise constant and smooth scenarios, respectively. In both scenarios, the PC solution converges to the optimal offline counterpart and, for the piecewise constant, also to the batch solution(s). This demonstrates the adaptive nature of Algorithm~\ref{alg:complete} to react to changes in the data statistics. While for the piecewise constant scenario PC and CC offer the same convergence speed (which is expected, as explained in \textit{``Does prediction help?''}), for the smooth scenario, the PC algorithm exhibits a faster convergence with respect to the prediction-free competitors CO and CC. This is because the temporal variability of the function (and of its gradient) is captured by the prediction step and exploited to fasten the convergence. 

%
%
\smallskip\noindent\textit{SEM.} Similar considerations hold for the TV-SEM, whose results are illustrated in Fig.~\ref{ctop} and Fig.~\ref{cbot}. In both scenarios, PC and CC offer the same convergence rate (which also converge to the batch solution for the constant scenario), faster than a CO and SGD implementation. Interestingly, after the triggering event at $T/2$, SGD can track the optimal solution faster than CO with performances similar to PC and CC. A possible justification may be the memory-less nature of SGD, i.e., it only considers the last sample for the gradient evaluation, thus discarding past data. This  renders the SGD more reactive to adapt to sudden changes of the data statistics compared to the memory-aware alternatives, which however exhibit similar performances thanks to the extra iteration they can benefit.
%
%

\smallskip\noindent\textit{SBM.} Finally, the TV-SBM results are shown in Fig.~\ref{rtop} and Fig.~\ref{rbot}. Also in this case, the PC solution converges to the offline counterpart for the two scenarios and faster than the prediction-free versions of the algorithm CC and CO. In particular, while in the piecewise constant scenario  PC converges faster than CC and the rank-one implementations, in the smooth scenario the rank-one implementations exhibit faster convergent behavior with respect to the non-stochastic implementations. Similar to what has been said for the TV-SEM results, a possible reason can be the memory-aware characteristics of the non-stochastic methods; that is, while the information present in past data can be beneficial in the static scenario and thus help PC and CC  to have a more reliable estimate of the true underlying (static) covariance matrix (and of the gradient), it may slow down the process in non-stationary environments with time-varying covariance matrices as in the smooth scenario.

\smallskip\noindent \textit{Required time.} 
An important metric to consider in time-sensitive applications is the  average time per iteration. We report this information in Table~\ref{tab:time}, for the PC step and  CVX, relative to the three considered models and settings in the top row of Fig.~\ref{fig:NMSE}. 
\begin{table}[t]
    \centering
        \caption{Average time (expressed in seconds) required to compute the PC and the CVX solution at each time instant.}
    \label{tab:time}
    {\normalsize
\begin{tabular}{ p{2.5cm}||p{2.5cm}|p{2.5cm}  }
 \hline
 &PC & CVX\\
 \hline
 TV-GGM   & $0.110 \times 10^{-2}$    & $3.6$\\
 TV-SEM &   $0.824 \times 10^{-2}$  & $2.0$   \\
TV-SBM & $0.023 \times 10^{-2}$ & $3.6$\\
 \hline
\end{tabular}
}
\vspace{-.3cm}
\end{table}
\noindent Combining the information of the table and that of the plots in Fig.~\ref{fig:NMSE}, it is clear how trading off the knowledge of the optimal solution for savings in terms of time seems an excellent compromise. Each prediction-correction step requires indeed around three orders of magnitude less time than the CVX counterpart, leading to a NSE at least smaller than $10e-1$.

\smallskip\noindent\textit{Does prediction help?} Notice how in the piecewise constant scenario, the PC strategy does not seem to offer a major advantage with respect the CC strategy. Although this behavior could be hypothesized (since the setting is static), it is here empirically confirmed. To can gain more insights we look at the structure of the prediction step (e.g., \eqref{eq:prediction-ggm}), where the components playing a role in the descent direction are: the gradient $\nabla_\bbs f(\cdot)$; the Hessian $\nabla_{\bbs \bbs} f(\cdot)$; and the time-derivative of the gradient $\nabla_{t\bbs} f(\cdot)$. Since we use $P=1$, i.e., only one prediction step, the term $(\hat{\bbs}^p -\hat{\bbs}_t=\bb0)$ that multiplies the Hessian does not contribute to the descent step. The added value of the prediction step with respect to a general (correction) descent method, in this case, would be only provided by the time-gradient $\nabla_{t\bbs} f(\cdot)$ (since the gradient $\nabla_\bbs f(\cdot)$ is common to either the prediction and the correction step). In the piecewise constant scenario, however, the underlying (true) covariance matrix is time-invariant within the two stationary intervals, leading to a zero time-derivative of the gradient (cf.~\eqref{eq: time-derivative-ggm}). This means that in static scenarios, with $P=1$, the prediction step boils down to a correction step. Differently, for $P=2$, the contribution of the second-order information may speed up the convergence, as illustrated in Fig.~\ref{fig:GGM_multiP} for TV-GGM, with respect to a correction-only algorithm using $C=3$.

\begin{figure}[h!]
    \centering
    \begin{subfigure}{0.25\textwidth}%
    \includegraphics[scale=0.3]{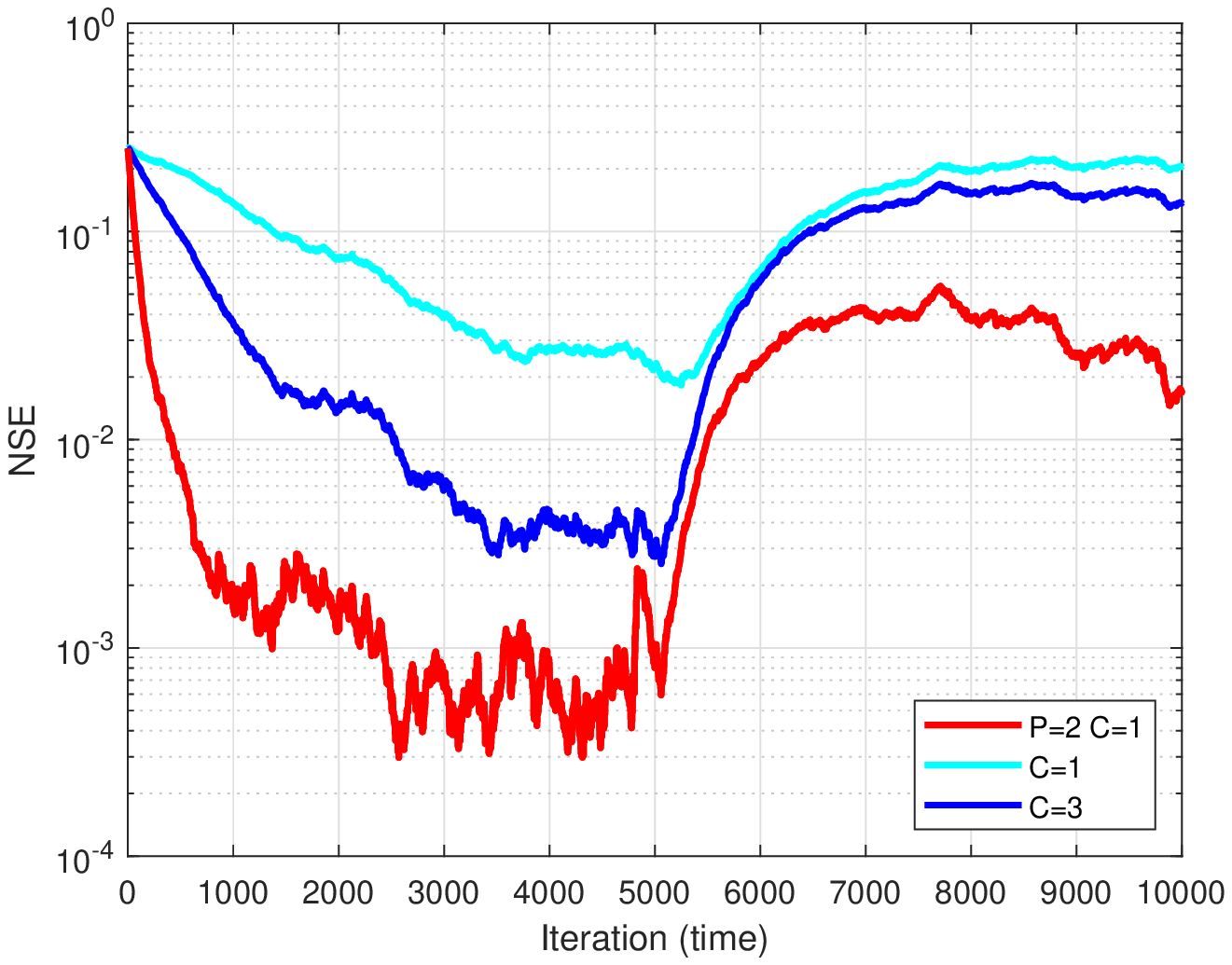}%
    \caption{}%
    \label{fig:GGM_multiP}%
    \end{subfigure}%
    \begin{subfigure}{0.25\textwidth}%
    \vspace{.2cm}
    \includegraphics[scale=0.3]{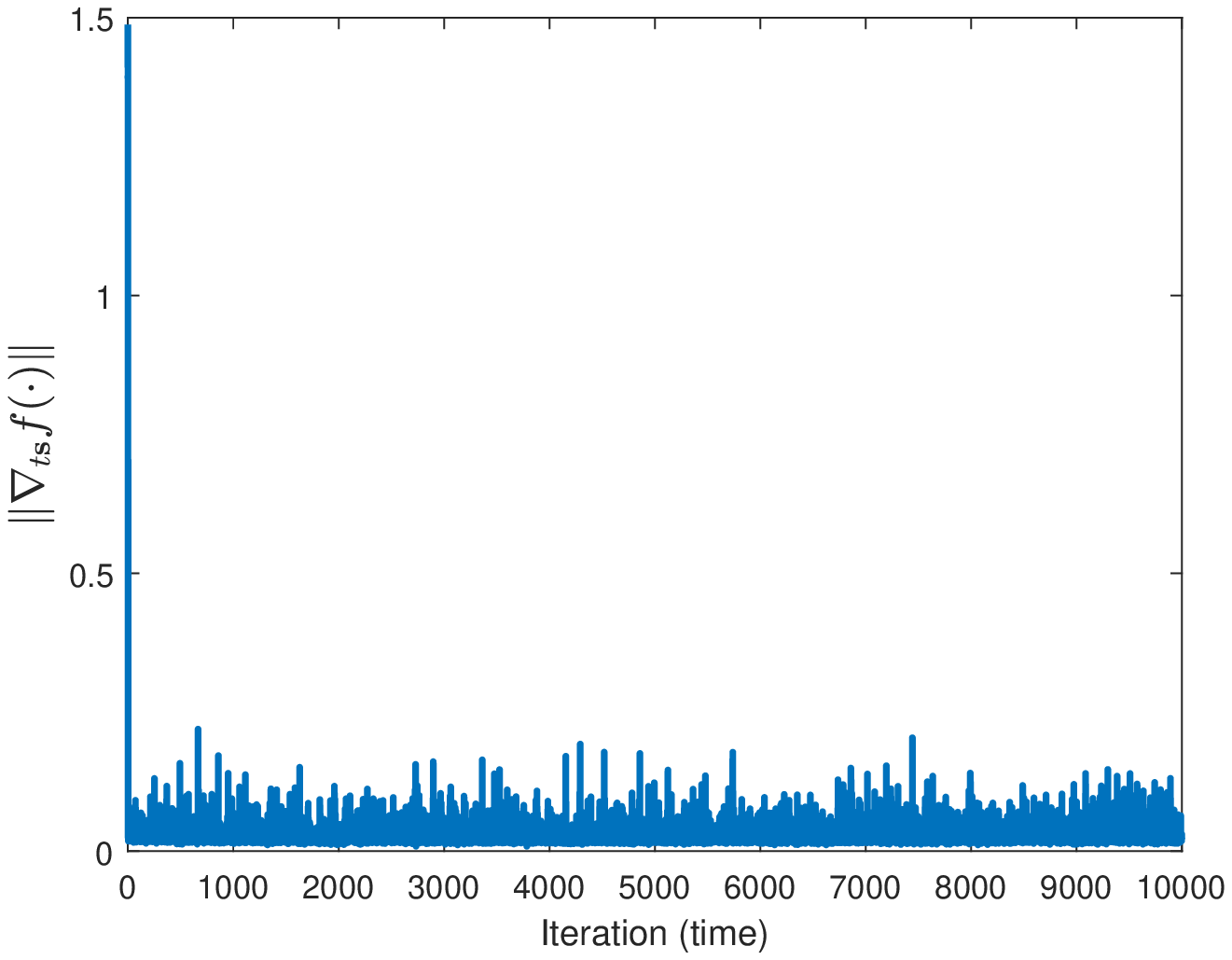}%
    \caption{}
    \label{fig:time-derivative}%
    \end{subfigure}%
    \label{fig:misc}
    \caption{(a): NSE of PC with $P=2$ and $C=1$, CO with $C=1$ and CC with $C=3$ for the piecewise constant scenario; (b) Norm of the time-derivative of the gradient as a  function of the iteration index for the smooth scenario. }
\end{figure}

In the smooth scenario, the temporal variability of the gradient captured by the time-derivative of the gradient $\nabla_{t\bbs}f(\cdot)$, plays a role in the prediction step, which can improve the convergence speed of the algorithm. The (bounded) norm of this vector over time is illustrated  in Fig.~\ref{fig:time-derivative} for the TV-GGM smooth scenario of Fig.~\ref{lbot}; this norm is linked to the constant $C_0$ introduced in Assumption~\ref{ass2} and the error in \eqref{eq:instantaneous-error}.

\smallskip\noindent All in all, the results indicate the convergence of Algorithm~\ref{alg:complete} to the optimal offline counterpart and its capability to  track it in non-stationary environments. The algorithm also converges to the batch solutions of the two stationary intervals, obtained with all the relevant data. A defining characteristic of Algorithm~\ref{alg:complete}  is its ability to naturally enforce similar solutions at each iteration, achieved with an early stopping of the descent steps, governed by the parameters $P$ and $C$. That is, the algorithm adds an implicit \textit{temporal} regularization to the problem which needs to be explicitly added when working with the entire batch of data. 

\smallskip\noindent Given these results and insights, we can outline a few principles that can be adopted when considering  Algorithm~\ref{alg:complete} for learning problems:

\begin{itemize}
    \item The prediction step with $P=1$ can be beneficial when the underlying data statistics change over time, so that the time-variability of the gradient can be exploited. Otherwise, in a complete static scenario, it coincides with a correction step.
    
    \item Increasing $P$ can improve the convergence speed when the approximated cost function is a good surrogate of the cost function in the next time instant.
    
    \item Memory-less (stochastic) variants of the algorithm can be suitable in fast-changing environments, due to their ability to discard past information and  react quickly to changes in data statistics.
    \end{itemize}

Being confident on the convergence of the algorithm, we now corroborate its performance with real data.

\subsection{Real Data}
\label{subsection:real data}
We now test the three considered algorithms on real data. Among other indicators employed in the simulations to assess the performance of the algorithm, we use the graph temporal deviation $\text{TD}(t):= \|\hat{\bbs}_t - \hat{\bbs}_{t-1}\|_2$, which measures the global variability on the edges of the graph for different time instants. To gain further insights on the network evolution over time, we consider additional metrics (such as number of edges and temporal gradient norm) and visual analysis tools which will be introduced in the application-specific scenario at hand. In this case, the hyperparameters of each function are chosen in such a way that the inferred graphs  are  neither trivial nor complete, and interpretable patterns consistent with real events are visible from the plots of the employed metrics.

\smallskip\noindent \textbf{TV-GGM for Stock Price Data Analysis.} 

\textit{Data description:} we collect historical stock (closing) prices relative to the S\&P500 Index for seven pharmaceutical companies over the time period August $12$th $2019$ to August $10$th $2021$ using~\cite{yahoo}. The collected data include the economic crisis related to the COVID-19 pandemic, followed by the vaccination campaign.  The companies of interest are Pfizer (PFE), Astrazeneca (AZN), Johnson \& Johnson (JNJ), GlaxoSmithKline (GSK), Moderna (MRNA), Novavax (NVAX) and Sanofi (SNY). Our goal is to leverage the TV-GGM in order to explore the relationships among these companies over time and observe the possible structural changes due to market instabilities.

\begin{figure*}%
\centering
\begin{subfigure}{0.33\textwidth}
\centering
\captionsetup{justification=centering}
\includegraphics[width=0.95\textwidth, height=5cm, trim =0.7cm 0.2cm 1cm 0.5cm , clip=true, keepaspectratio=true]{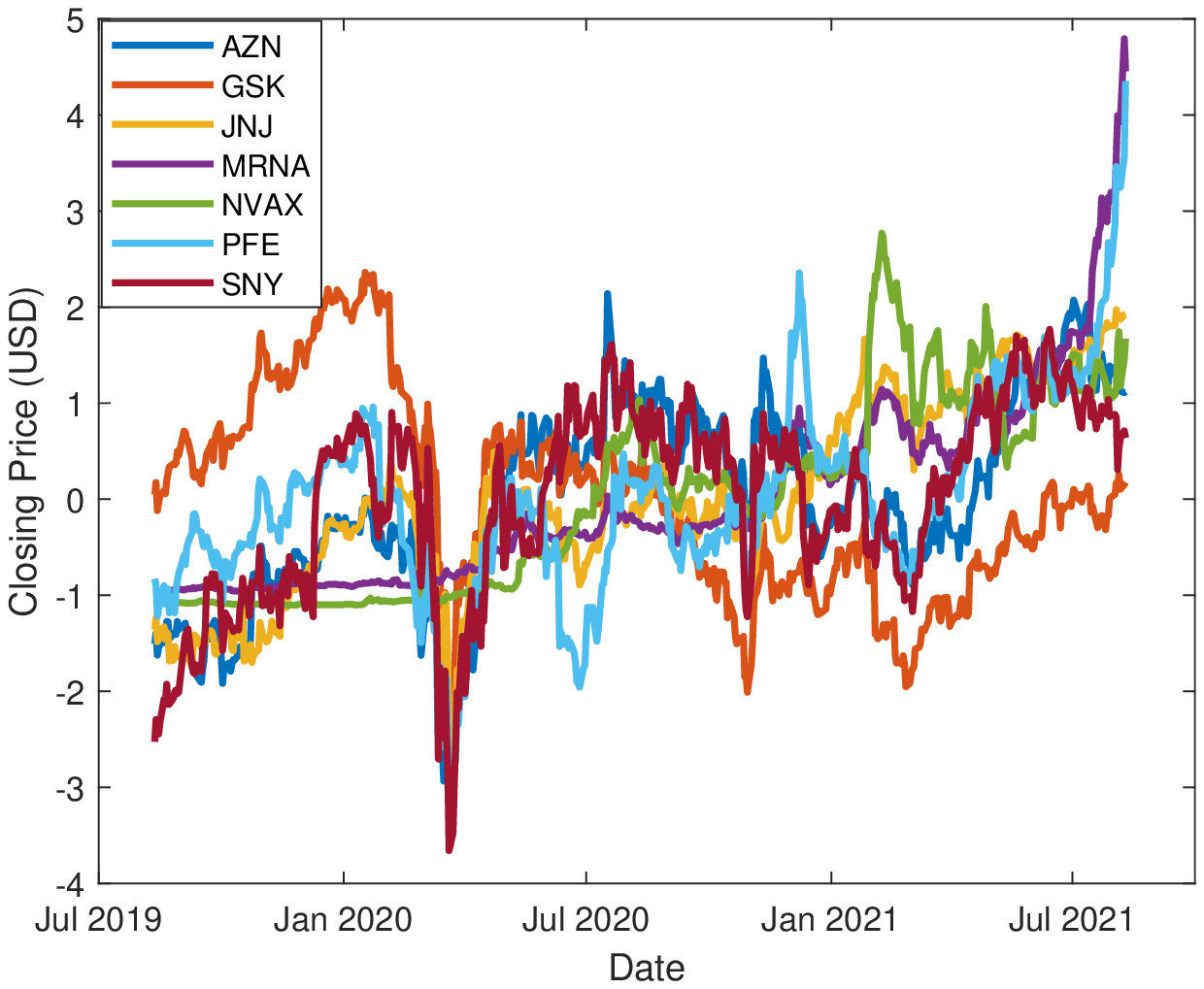}%
\vspace{-.2cm}
\caption{}%
\label{fig:logprice}%
\end{subfigure}%
\begin{subfigure}{0.33\textwidth}
\centering
\captionsetup{justification=centering}
\includegraphics[width=0.95\textwidth, trim= 1cm 0cm 1cm 0.3cm, clip=true,  keepaspectratio=true]{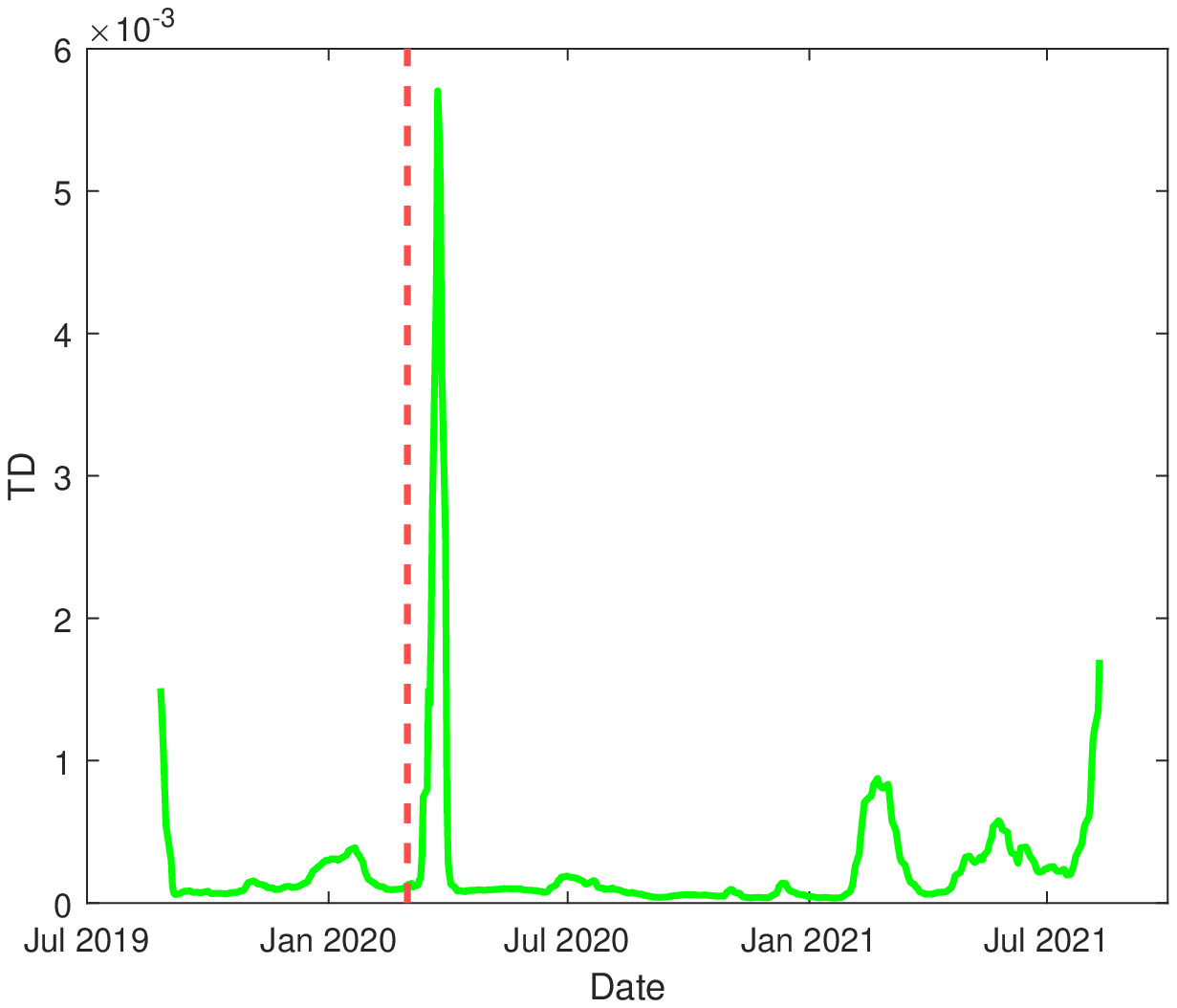}
\vspace{-.2cm}
\caption{}%
\label{fig:TD_Pharma}%
\end{subfigure}%
\begin{subfigure}{0.33\textwidth}
\centering
\captionsetup{justification=centering}
\includegraphics[width=0.95\textwidth]{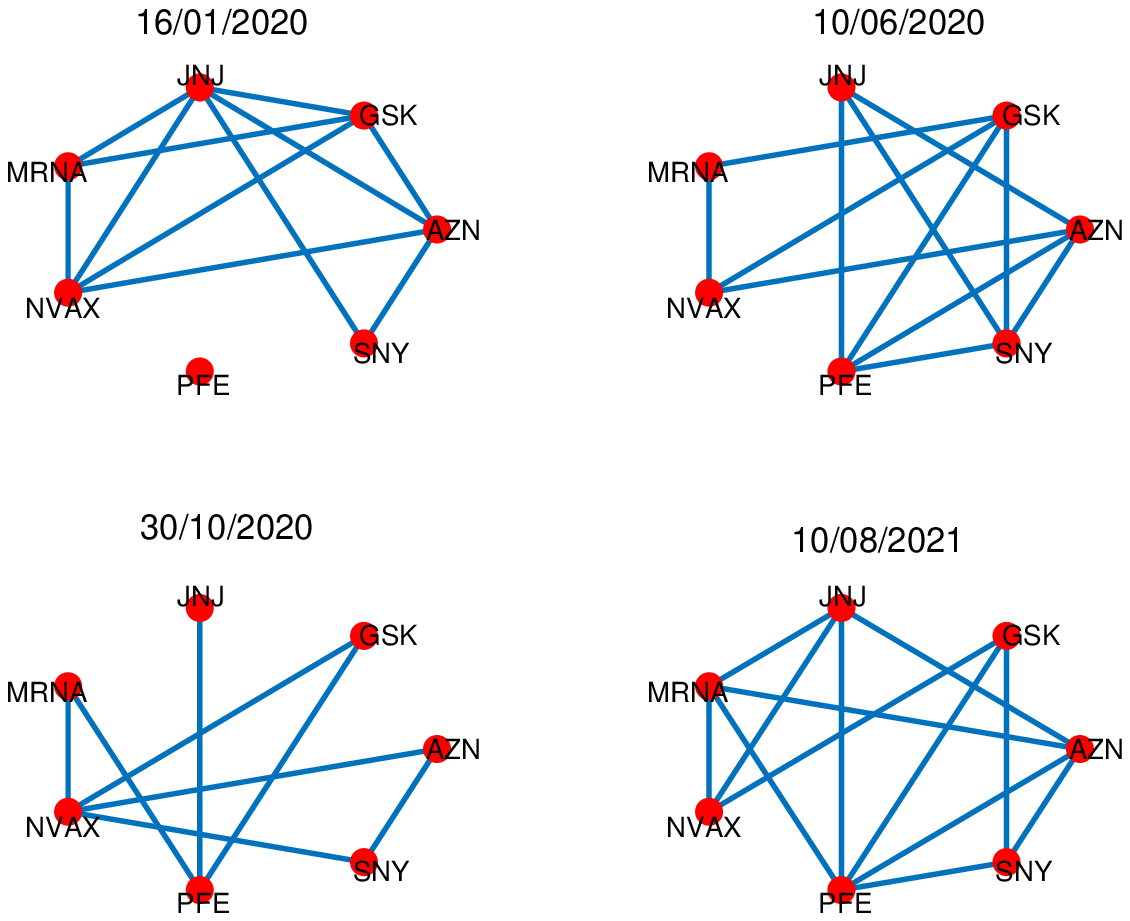}%
\vspace{-.2cm}
\caption{}%
\label{fig:graphs_pharma}%
\end{subfigure}%
\vspace{-.2cm}
\caption{(a) Standardized time series for the period  August 12th 2019 -  August 10th 2021; (b) graph temporal deviation for the stock market graph inferred with TV-GGM. The sharp peaks around March 2020 and after January 2021 happen consistently with real events; (c) inferred topologies at four different dates of interest. The absence of an edge between two nodes indicates their conditional independence.}
\label{fig:pharma_all}\vspace{-5mm}
\end{figure*}

 \textit{Results:} We consider $T = 504$ measurements (working days in August 2019 -  August 2021) as graph signals $\{\bbx_t\}$ for the $N=7$ quantities of interest, which are further standardized, i.e., each variable is centered and divided by its empirical standard deviation; see Fig.~\ref{fig:logprice} for a plot of the standardized time series. We run the TV-GGM algorithm for different values of the forgetting factor $\gamma$, and monitor the evolution of the metrics earlier introduced. The value $\gamma = 0.75$ yielded results most consistent with the data behavior. 
 
It is clear from  Fig.~\ref{fig:logprice} and the TD indicator in Fig.~\ref{fig:TD_Pharma} that around March 2020 and after January 2021 the market has changed significantly, due to the instability generated by the pandemic and by the follow-up starting vaccination campaign. The sharp peaks in Fig.~\ref{fig:TD_Pharma} around around the same period are a consequence of the dynamic inter-relationships among the companies; the inferred graph changes substantially in the two periods of interest and TD captures the market variability.

 To really enjoy the visualization potential offered by graphs as a tool, we show in Fig.~\ref{fig:graphs_pharma} snapshots of the inferred time-varying graph at four different dates of interest. Common among the four graphs is the presence of the edge connecting MRNA and NVAX, and the edge connecting AZN and SNY. The pharmaceutical companies associated to the endpoints of each of these two edges also show a similar trend in Fig.~\ref{fig:logprice}. Notice moreover that since the sparsity pattern of the precision matrix reveals conditional independence among the variables indexed by its zero entries, these graphs enable us to visually inspect such independence over time. Although the information endowed in these graphs may carry a financial significance, we leave this possible knowledge-discovery task out of this manuscript, to avoid misleading or erroneous interpretations.

\smallskip\noindent \textbf{TV-SEM for Temperature Monitoring.}

\textit{Data description:} for this experiment we consider the publicly available weather dataset\footnote{Data available at \hyperlink{https://www.met.ie/climate/available-data/historical-data}{https://www.met.ie/climate/available-data/historical-data}} provided by the Irish Meteorological Service, which contains hourly temperature (in $^{\circ}$C) data from 25 stations across Ireland. We monitor the temperature evolution over the sensor network for the period January 2016 to May 2020, and leverage the TV-SEM to infer the time-varying features of the graph learned by the algorithm.

\textit{Results:} for the analysis we consider $T=38713$ measurements as graph signals $\{\bbx_t\}$ for the $N=25$ stations under consideration, standardizing the data as done in the previous experiment;  Fig.~\ref{fig:ireland_time_series} depicts the standardized time series. It is interesting to notice the sinusoidal-like behavior of the aggregate time-series, due to higher (lower) temperature during the summer (winter) period, resulting in a smooth signal profile. 

 \begin{figure}[h]
\begin{subfigure}{0.95\columnwidth}
\centering
\includegraphics[width=\textwidth, trim= 1cm 0cm 1cm 0cm , clip= true]{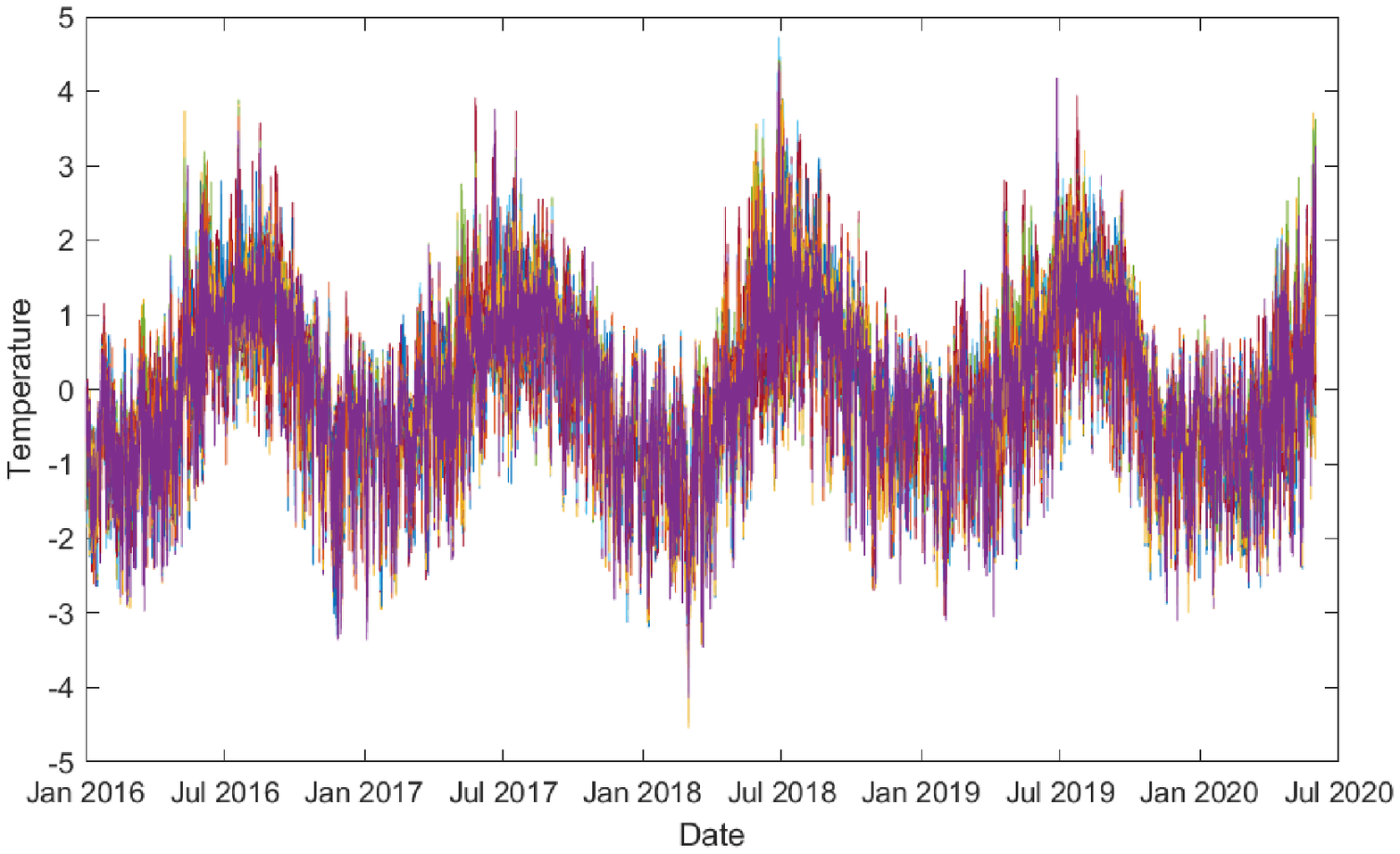}%
\vspace{-.2cm}
\caption{}%
\label{fig:ireland_time_series}%
\end{subfigure}
\begin{subfigure}{0.95\columnwidth}
\centering
\includegraphics[width=\textwidth, height=5cm, keepaspectratio=true]{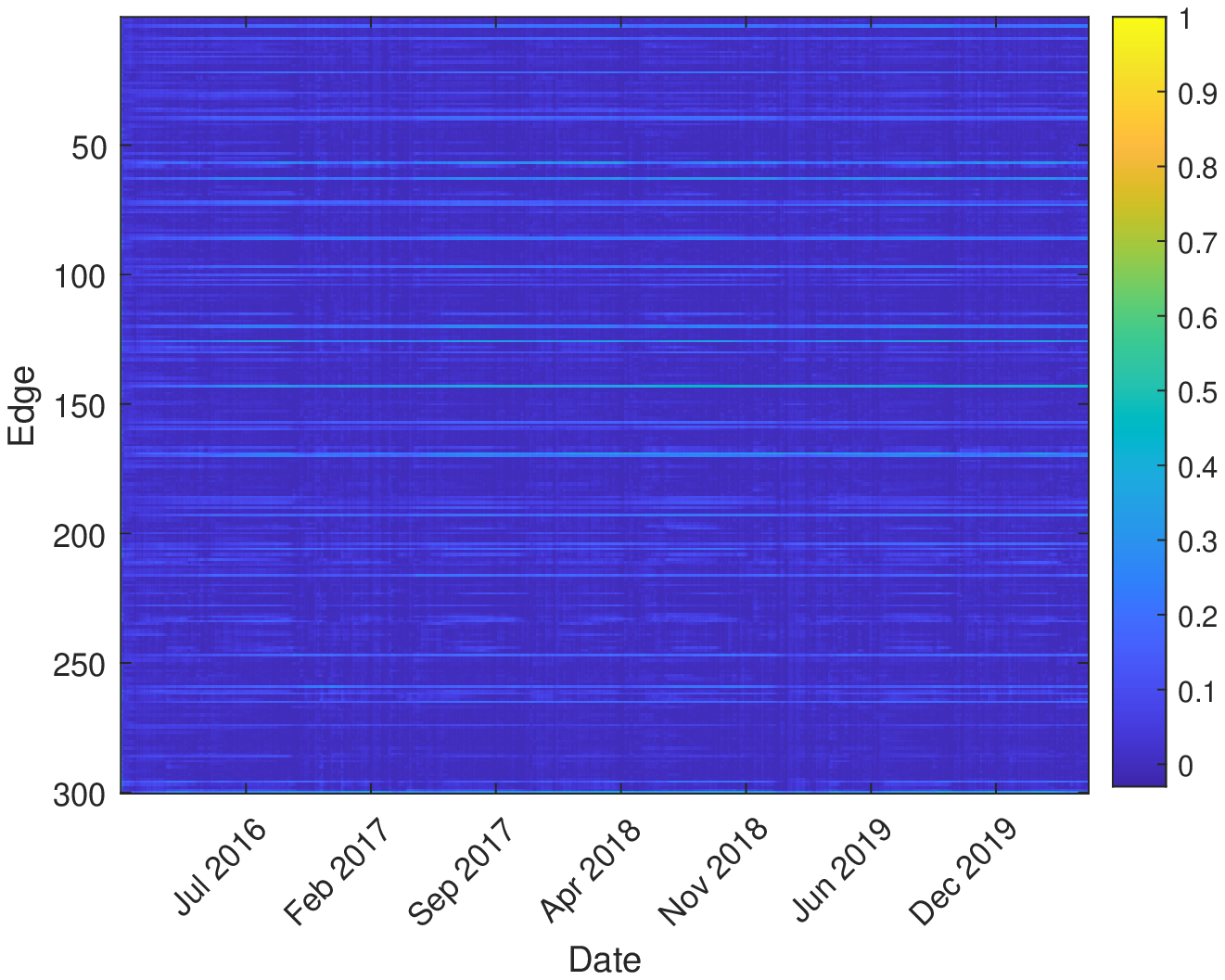}%
    \caption{}
    \label{fig:edgePattern_ireland}
\end{subfigure}
\caption{ (a) Standardized time series for the 25 Irish weather stations and (b) evolution of each edge weight over time.}
\end{figure}

 Fig.~\ref{fig:edgePattern_ireland} illustrates the  sparsity pattern of the time-varying graph and the importance of the weights at every time instant. This \textit{learn-and-show} feature offered by Algorithm~\ref{alg:complete} gives us the ability to visualize the learning behavior of the algorithm on-the-fly, a strength of low-cost iterative algorithms w.r.t. batch counterparts. From the figure (and the observed almost zero graph temporal deviation, which is not illustrated here)  a consistent temporal homogeneity  is visible, i.e., the graph does not change significantly over adjacent time instants. In other words, nodes influencing each other in a particular time instant, are likely to influence each other in other time instants. A reasonable explanation is given by the smooth and regular pattern exhibited by the time-series of Fig.~\ref{fig:ireland_time_series}, which
  is a consequence of the meteorological similarity over time, and by their high correlation coefficient.

An interesting trend arises when observing the number of edges of the graph inferred over time, shown in Fig.~\ref{fig:numberedges_ireland}. Although in adjacent time instants the number does not change abruptly, a pattern can be identified over a longer time span. In particular, during winter and summer there is a sharp increment  in the number of edges, with respect to autumn and spring where there is a significant reduction. To ease the visualization, the vertical red lines are placed in correspondence of the winter period of every year, while blue lines in correspondence of the autumn period. A possible reason for this phenomenon is given   by the reduced variability of the temperature among the stations during  summer and winter, and a higher variability during spring and autumn, leading to different graphs.

 \begin{figure*}%
  \centering
   \captionsetup{justification=centering}
  \begin{subfigure}{0.4\textwidth}
    \includegraphics[width=.95\textwidth, trim= 0 0 0 2cm]{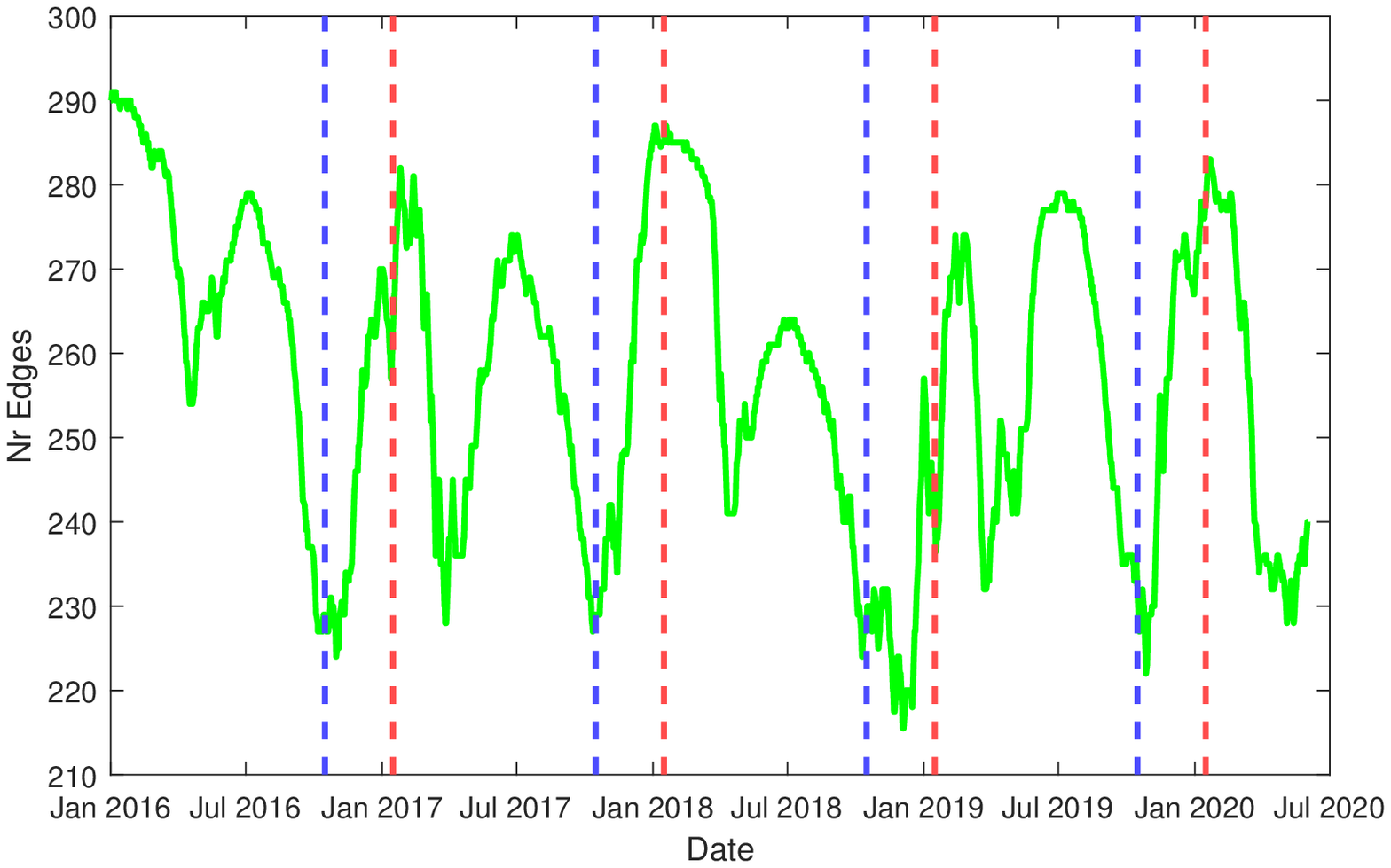}
    \caption{}
    \label{fig:numberedges_ireland}
    \end{subfigure}%
    \centering
    \begin{subfigure}{0.3\textwidth}
    \centering
    \includegraphics[width=.95\textwidth]{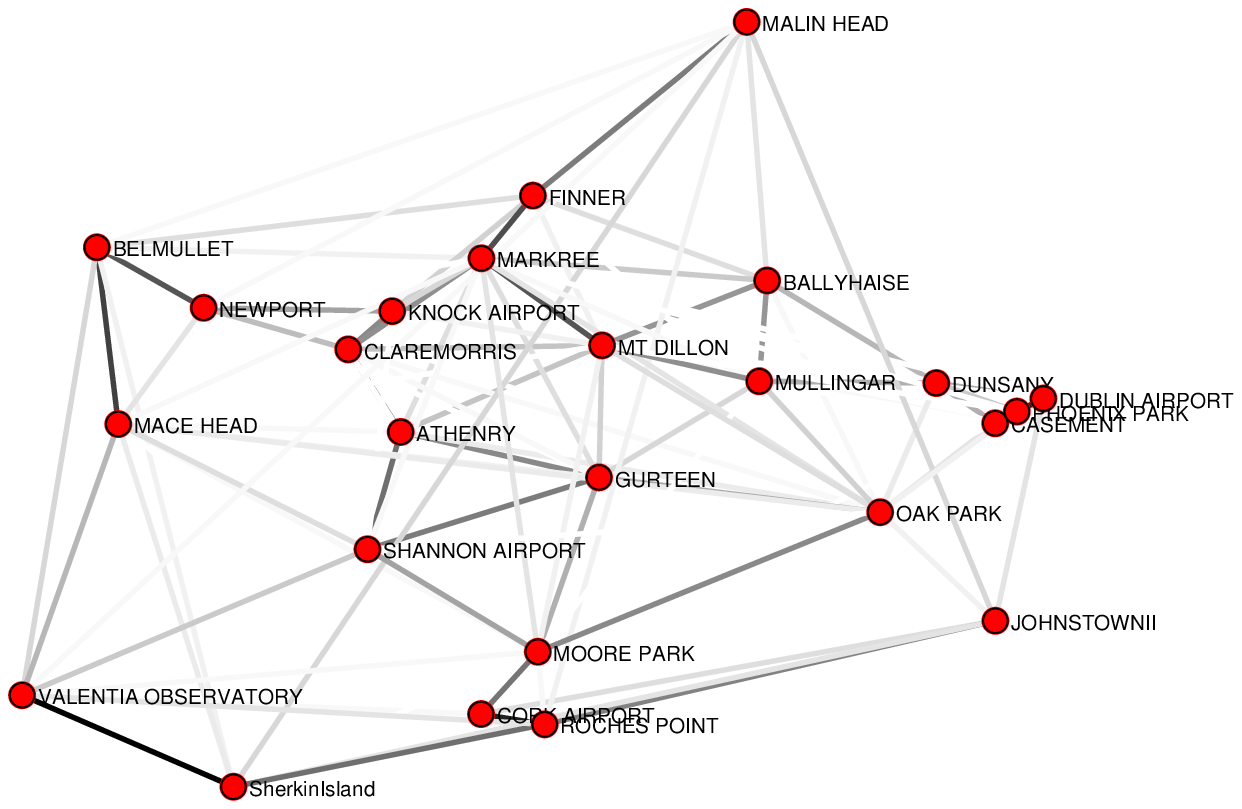}%
    \caption{Oct 2016}%
    \label{autumn_graph}%
    \end{subfigure}%
    \begin{subfigure}{0.3\textwidth}
    \centering
    \captionsetup{justification=centering}
    \includegraphics[width=.95\textwidth]{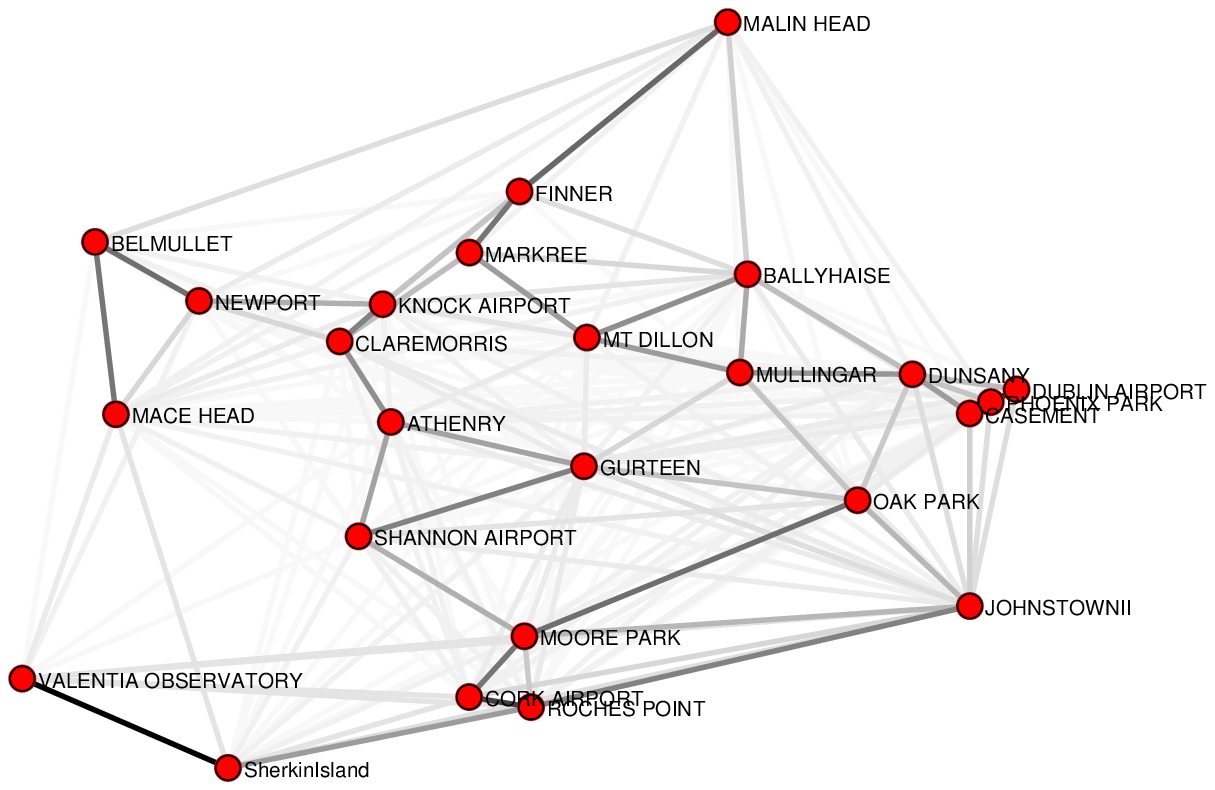}
    \caption{Jan 2017}%
    \label{winter_graph}%
    \end{subfigure}%
\caption{Ireland temperature dataset. (a) Number of edges of the inferred graph over time. The red vertical lines correspond to January 15 of each year (winter), while the blue vertical line correspond to October 15 (autumn); snapshot of the inferred time-varying graph during (b) October 2016 (autumn graph) and (c) January 2017 (winter graph). Notice how stations close in space tend to be connected.}
\label{fig:ireland_dataset}\vspace{-5mm}
\end{figure*}

For the sake of visualization, we also report the inferred graphs for October 2016 (autumn) and January 2017 (winter). In line with our previous comments regarding Fig.~\ref{fig:numberedges_ireland}, a  lower number of edges is visible in the autumn graph with respect to the winter graph; in particular, edges present in the autumn graph are also present in the winter one. Finally, notice how stations close in space tend to be connected, thus showing how stations close to each other have a greater influence  with respect to stations farther away in space.

\smallskip\noindent \textbf{TV-SBM for Epileptic Seizure Analysis.}

\textit{Data description:} we use electrocorticography (ECoG) time series collected during an epilepsy study at the University of California, San Francisco (UCSF) Epilepsy Center, where an 8 × 8 grid of electrodes was implanted on the cortical brain's surface of a 39-year-old  woman with medically refractory complex partial seizure~\cite{kramer2008emergent}. The grid was supplemented by two strips of six electrodes: one deeper implanted over the left suborbital frontal lobe and the other over the left hippocampal region, thus forming a network of $76$ electrodes, all measuring the voltage level in proximity of the electrode, which is an evidence of the local brain activity. The sampling rate is $400$ Hz and the measured time series contains the 10 seconds interval preceding the seizure  (pre-ictal interval) and the 10 seconds interval after the start of the seizure  (ictal interval). Our goal is to leverage the TV-SBM in order to explore the dynamics among different brain areas at the seizure onset.

\textit{Results:} for our analysis we consider $T=3200$ time instants as graphs signals $\{\bbx_t\}$ for the $N=76$ electrodes, which are further filtered (over the temporal dimension) at $\{60, 180\}$Hz to remove the spurious power line frequencies, and standardized as explained in the previous experiments.

Fig.~\ref{fig:TD_epilepsy} shows the graph temporal deviation, where we observe an increasing and protracted  variability of the TD shortly after the seizure onset (red vertical line), proving TD to serve as an indicator of network alteration suitable for time-varying scenarios. To visualize the on-the-fly  learning behavior of the algorithm, in Fig.~\ref{fig:numberedges_epilepsy} we show the evolution of (a fraction\footnote{For visualization, we show $500$ random edges, since we recall that the number of total edges in an undirected graph of $N$
nodes is $N(N-1)/2$.} of) the edge weights over time. In the first half of the time-horizon, we notice the presence of stronger edges with respect to the second half, where the graph is sparser. 
We show two snapshots of the time-varying graph in Fig.~\ref{fig:graphEpilepsy}, for the time instants $1500$ (pre-ictal) and $1800$ (ictal), where we also report the closeness centrality of each node, which expresses how ``close'' a node is to all other nodes in the network (calculated as the average of the shortest path length from the node to every other node in the network). During the ictal interval, the graph tends to be more disconnected and its nodes to have a lower closeness centrality value, especially in the lower part of the graph. In addition, we observe how the number of (strong) edges and the closeness centrality value drop in the ictal graph, especially in the lower part of the graph. This is consistent with the findings in \cite{kramer2008emergent} and indicates that, on average, signals in the pre-ictal interval behave more similar to each other as opposed to the signals in the ictal interval.

\begin{figure}[h!]
\centering
\includegraphics[width=0.95\columnwidth, height= 4cm, keepaspectratio=true]{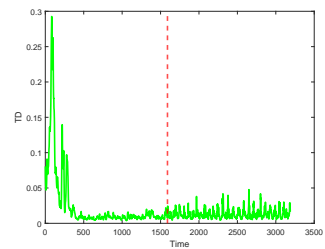}%
\vspace{-.1cm}
\caption{Graph temporal deviation for the epilepsy study. The red line indicates the seizure onset. During the ictal interval, a higher temporal deviation can be observed, indicating that the inferred graph is changing substantially.}%
\label{fig:TD_epilepsy}
\vspace{-5mm}
\end{figure}

 \begin{figure*}%
  \centering
   \captionsetup{justification=centering}
  \begin{subfigure}{0.33\textwidth}
    \includegraphics[width=.95\textwidth]{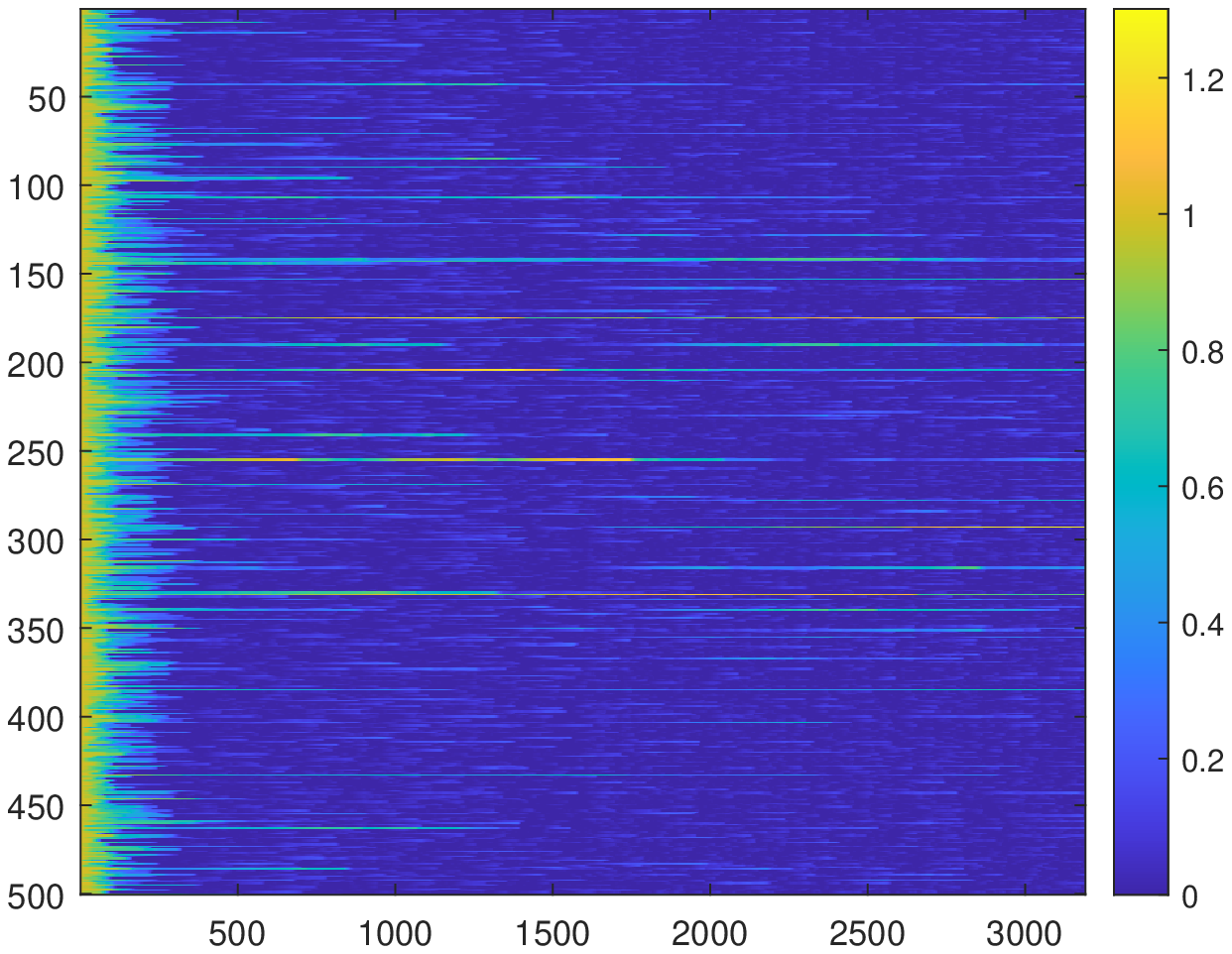}
    \caption{}
    \label{fig:numberedges_epilepsy}
    \end{subfigure}%
    \centering
    \begin{subfigure}{0.7\textwidth}
    \centering
    \includegraphics[scale=0.5, trim= 1cm 1cm 1cm 0cm, clip=true]{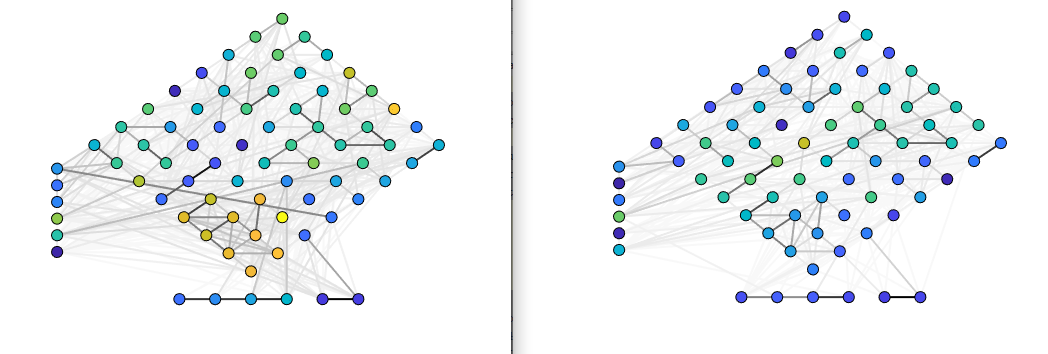}%
    \caption{}%
    \label{fig:graphEpilepsy}%
    \end{subfigure}%
\vspace{-.2cm}
\caption{Epilepsy dataset. (a) evolution of each edge weight over time;  (b)  snapshots of the inferred time-varying graph at time instant 1500 and 1800. The color of an edge indicates its weight, with darker colors indicating higher weights, while the color of a node indicates the closeness centrality of such node, with brighter colors indicating higher values of closeness centrality.}
\vspace{-5mm}
\end{figure*}

\section{Conclusion}
In this manuscript, we proposed an algorithmic template to learn time-varying graphs from streaming data. The abstract time-varying graph learning problem, where the data influence is expressed through the empirical covariance matrix, is casted as a composite optimization problem, with different terms regulating different desiderata. The framework, which works in non-stationary environments, lies upon novel iterative time-varying optimization algorithms, which on one side exhibit an implicit temporal regularization of the solution(s), and on the other side accelerate the convergence speed by taking into account the time variability. We specialize the framework to the Gaussian graphical model, the structural equation model, and the smoothness-based model, and we propose ad-hoc vectorization schemes for structured matrices central for the gradient computations which also ease storage requirements. The proposed approach is accompanied by theoretical performance guarantees to track the optimal time-varying solution, and is further validated with synthetic numerical results. Finally, we  learn time-varying graphs in the context of stock market, temperature monitoring, and epileptic seizures analysis. The current line of work can be enriched by specializing the framework to other static graph learning methods present in literature, possibly considering directed graphs, by implementing distributed versions of the optimization algorithms, and by applying the developed models in other real-world applications.

\appendices

\section{}
\label{sec:app-a}
Consider the multi-valued function $\ccalT: \reals^N \to \reals^N$, which we will refer to as \textit{operator}. Here, we briefly review some operator theory concepts used in this manuscript; see \cite{bauschke2011convex}.

\noindent \textbf{Projection operator.} Given a point $\bbx \in \mathbb{R}^N$, we define projection  of $\bbx$ onto the convex set $\ccalC \subseteq \mathbb{R}^N$ as:
\begin{align}
    \mathbb{P}_{\ccalC}(\bbx):=  \argmin_{\bbz \in \ccalC} \frac{1}{2} \|\bbz - \bbx\|_2
\end{align}

\noindent \textbf{Proximal operator.} Consider the convex function $g: \mathbb{R}^N \to \mathbb{R}$. We define  the proximal operator of $g(\cdot)$, with penalty parameter $\rho > 0$, as:
\begin{align}
    \prox_{g,\rho}(\bbx):= \argmin_\bbz \{g(\bbz) + \frac{1}{2 \rho} \|\bbz -\bbx\|_2^2\}
\end{align}
For some functions, the proximal operator admits a closed form solution \cite[Ch.~6]{beck2017first}. In particular:
\begin{itemize}
    \item if $g(\bbx)= \iota_{\ccalC}(\bbx)$ then $\prox_g(\bbx)= \mathbb{P}_\ccalC(\bbx)$, i.e., it is the projection of $\bbx$ onto the convex set $\ccalC$.
    \item if $g(\bbx)= \lambda \|\bbx\|_1$ then $\prox_g(\bbx)= \operatorname{sign}(\bbx) \odot [\bbx - \lambda\boldsymbol{1}]_+ $, i.e., it is the soft-thresholding operator.
\end{itemize}

Consider the convex minimization problem:
\begin{align}
\label{eq:com}
    \min_\bbx f(\bbx) + g(\bbx)
\end{align}
with $f, g: \mathbb{R}^N \to \mathbb{R}$ convex. It can be shown that problem \eqref{eq:com} admits at least one solution~\cite{combettes2005signal}, which can be found by the fixed point equation:
\begin{align}
    \bbx= \prox_{g,\rho}(\bbx- \rho \nabla f(\bbx))
\end{align}

\section{}
\label{proof:bounds}
\begin{proof}[Proof of Claim 1: TV-GGM]
Recall the expression of the Hessian in \eqref{eq:hessian-ggm}, i.e., $\bbH (\bbS)= \bbD^{\top} (\bbS \otimes \bbS )^{-1} \bbD$ and that matrix $\bbS \in \ccalS$ is the precision matrix, with $\ccalS=\{\bbS \in \mathbb{S}_{++}^N \vert \xi \bbI \preceq \bbS \preceq \chi \bbI \}$. 

For the strong convexity, notice that since $\bbS \succ 0$, then also $\bbH(\bbS) \succ0$. Indeed, by exploiting the semi-orthogonality of matrix $\bbD/\sqrt{2}$, we have:
\begin{align}
     &\lambda_{\text{min}}(\bbH(\bbS)) = \min_{\|\bbx\|=1} \bbx^\top \bbD^{\top} (\bbS \otimes \bbS )^{-1} \bbD \bbx  \\ 
     \nonumber &\geq  \min_{\|\bbx\|=1} \bbx^\top \frac{\bbD^{\top}}{\sqrt{2}} (\bbS \otimes \bbS )^{-1} \frac{\bbD}{\sqrt{2}} \bbx
     =  \min_{\|\bby\|=1} \bby^\top (\bbS \otimes \bbS )^{-1} \bby
     \\\nonumber
     &= \min_{\|\bbz\|=1} \sum_{i=1}^N  \sum_{j=1}^N \frac{ z_iz_j}{\lambda_i(\bbS)\lambda_j(\bbS)} 
     \geq \frac{1}{\lambda_{\text{max}}^2(\bbS)}= 1/ \chi^2
\end{align}

For the Lipschitz continuity of the gradient,  we have
\begin{align}
    \|\bbD^\top (\bbS \otimes \bbS )^{-1}\bbD\| &\leq \|\bbD\|^2  \|(\bbS \otimes \bbS )^{-1} \|  \\
    = \nonumber 2 \|(\bbS \otimes \bbS )^{-1}\| 
    &= \nonumber 2 \|\bbS^{-1} \otimes \bbS^{-1}\| \\
    =  \nonumber 2 \sqrt{\lambda_{\text{min}}(\bbS)^{-2}}
    &= \nonumber 2/\xi
\end{align}
\end{proof}

\begin{proof}[Proof of Claim 2: TV-SEM]

Denote with $\lambda_{\text{min}}$ and $\lambda_{\text{max}}$ the smallest and highest eigenvalues for the set of empirical covariance matrices obeying the SEM model. Recall the expression of the Hessian in \eqref{eq:hessian-sem}, i.e. $\bbH (\bbS; t)= \bbQ_t$, where $\bbQ_t:=\bbD_h^\top(\empcov_t \otimes \bbI)\bbD_h$. Since $\bbD_h / \sqrt{2}$ is a semi-orthogonal matrix, we have:

\begin{align}
     &\lambda_{\text{min}}(\bbH(\bbS)) = \min_{\|\bbx\|=1} \bbx^\top \bbD_h^\top(\empcov_t \otimes \bbI)\bbD_h \bbx  \\ 
     \nonumber &\geq   \min_{\|\bby\|=1} \bby^\top (\empcov_t \otimes \bbI ) \bby = \min_{\|\bbz\|=1}\sum_{i=1}^N  \lambda_i(\empcov_t) z_i^2 \geq \lambda_{\text{min}}
\end{align}
where $\lambda_{\text{min}}$ is the smallest eigenvalue of $\empcov_t$.

For the Lipschitz continuity of the gradient, we have:
\begin{align}
    \|\bbD_h^\top(\empcov_t \otimes \bbI)\bbD_h\| &\leq 2 \|\empcov_t \otimes \bbI\| =  2\lambda_{\text{max}}
\end{align}
\end{proof}


\begin{proof}[Proof of Claim 3: TV-SBM]
For the strong convexity it suffices to notice that for $m >0$,  $f(\bbs;t) - \frac{m}{2} \|s\|^2 = \bbs^\top\bbz_t - \lambda_2 \boldsymbol{1}^\top \log (\bbK\bbs) + (\lambda_1 - \frac{m}{2}) \|\bbs\|^2$ is convex. In turn, this implies that strong convexity of $f(\cdot; t)$ is guaranteed for $0<m\leq 2\lambda_1$.

For the Lipschitz continuity of the gradient, recall that nodal degree vector $\bbd \succ 0$. Denote with $d_\text{min}$ the minimum degree of the GSO search space. Also, recall the expression of the Hessian $\bbH= \bbK^\top \Diag(\boldsymbol{1}\oslash (\bbK\bbs)^{\circ 2})\bbK$. Then:
\begin{align}
    \|\bbK^\top \Diag(\boldsymbol{1}\oslash (\bbK\bbs)^{\circ 2})\bbK\| & \! \leq \! \|\bbK\|^2 \max (\boldsymbol{1}\oslash (\bbK\bbs)^{\circ 2})) \\ \nonumber
    = \|\bbK\|^2 d_\text{min}^{-2} &= 2 (N-1) d_\text{min}^{-2},
\end{align}
where we made use of \cite[Lemma 1]{saboksayr2021online} for the bound of $\bbK$.
\end{proof}

%
%
\section{}
\label{app:complexity}
The computational (arithmetic) complexity per iteration of Algorithm~\ref{alg:complete} is dominated by the rank-one covariance matrix update in $\ccalO(N^2)$ and by the method-specific gradient computations involved in the prediction and correction steps (and eventually Hessian, if $P>1$ [cf. Section~\ref{sec:numerical-result} ``\emph{Does prediction help?}'' ]). Such method-specific computational complexities are shown next, together with a discussion on the costs for the offline counterparts.

\smallskip\noindent
\textbf{TV-GGM.}
The worst case scenario computational complexity of the gradient $ \nabla_\bbs f(\bbs; t)$ in~\eqref{eq:gradient-ggm} is $\ccalO(N^3)$, which is due to the matrix inversion. This cost might be lowered exploiting the sparsity pattern of the sparse triangular factor of $\bbS$ or, in our case, exploiting the fact that it is a small perturbation with respect to the previous iterate. The multiplication with matrix $\bbD^\top$ has a cost of $\ccalO(N^2)$, since  $\bbD \in \mathbb{R}^{N^2 \times N(N+1)/2}$ has at most two $1$'s in each column and exactly one $1$ in each row.

The worst case scenario computational complexity of the Hessian $\nabla_{\bbs \bbs} f(\bbs; t)$ in~\eqref{eq:hessian-ggm} would be $\ccalO(N^3)$. However, because the Hessian is used in a matrix-vector multiplication [cf. \eqref{eq:prediction-ggm}], its factorization leads to a cost for the prediction step of $\ccalO(N^3)$. Indeed, exploiting the Kronecker product, the Hessian can be written as $\bbD^\top(\bbS^{-1} \otimes \bbI_N)( \bbI_N \otimes \bbS^{-1})\bbD$; then, the multiplication of the Hessian for a vector simply entails the succession of four sparse matrix-vector multiplications all with a cost of $\ccalO(N^3)$.

The term $\nabla_{t \bbs} f(\bbs ; t)$ in~\eqref{eq: time-derivative-ggm} has a computational complexity of $\ccalO(N^2)$. Thus the overall computational complexity per iteration is $\ccalO(N^3)$.

\smallskip\noindent
\textbf{TV-SEM.} The overall cost is dominated by the computation of $\bbQ_t=\bbD_h^\top(\empcov_t \otimes \bbI)\bbD_h$, which is present in the gradients and the Hessian. The matrix-matrix multiplication(s) have a cost of $\ccalO(N^3)$, since  $\bbD_h \in \mathbb{R}^{N^2 \times N(N-1)/2}$ has at most two $1$'s in each column and exactly one $1$ in each row.
Thus the overall computational complexity per iteration is $\ccalO(N^3)$.

\smallskip\noindent
\textbf{TV-SBM.} Each column of $\bbK$ has exactly two non-zero entries (and each row has $N-1$ non-zero entries), thus $\bbK\bbs$ has a computational cost of $2|\ccalE|$, with $|\ccalE|$ the number of edges of the graph represented by $\bbS$ (in other words, $\|\bbs\|_0$).  The operation $\bbK^\top (\boldsymbol{1} \oslash \bbK\bbs)$ has a cost of $\ccalO(N^2)$. The computational complexity of the Hessian is $\ccalO(N^3)$, since it is the weighted sum of $N$ outer products of vectors which are $(N-1)$-sparse in the same positions.

Thus the overall computational complexity per iteration is $\ccalO(N^2)$ if $P=0,1$ and $\ccalO(N^3)$ if $P>1$.

\smallskip\noindent
\textbf{Offline.}
The computational complexity for each time instant $t$ incurred by an offline solver to solve instances of problem~\eqref{eq:time-varying}  depends by its algorithm-specific implementation closely related to the problem structure. The three problems we consider are (converted into) semidefinite programs (SDPs) and solved, in our case, by SDPT3, a Matlab implementation of infeasible primal-dual path-following algorithms, which involves the computation of second-order information. Since these computations are continuously repeated, for a fixed time instant $t$, till the algorithm convergence (say $I$ iterations), a trivial lower bound for computing the offline solution for the three considered problems is $\bbOmega(IN^3)$. To this cost must be also added the cost of other solver-specific steps which we do not explicitly consider here.

\bibliographystyle{IEEEtran}
\bibliography{refs}

\begin{thebibliography}{10}
\providecommand{\url}[1]{#1}
\csname url@samestyle\endcsname
\providecommand{\newblock}{\relax}
\providecommand{\bibinfo}[2]{#2}
\providecommand{\BIBentrySTDinterwordspacing}{\spaceskip=0pt\relax}
\providecommand{\BIBentryALTinterwordstretchfactor}{4}
\providecommand{\BIBentryALTinterwordspacing}{\spaceskip=\fontdimen2\font plus
\BIBentryALTinterwordstretchfactor\fontdimen3\font minus
  \fontdimen4\font\relax}
\providecommand{\BIBforeignlanguage}[2]{{%
\expandafter\ifx\csname l@#1\endcsname\relax
\typeout{** WARNING: IEEEtran.bst: No hyphenation pattern has been}%
\typeout{** loaded for the language `#1'. Using the pattern for}%
\typeout{** the default language instead.}%
\else
\language=\csname l@#1\endcsname
\fi
#2}}
\providecommand{\BIBdecl}{\relax}
\BIBdecl

\bibitem{natali2021online}
A.~Natali, M.~Coutino, E.~Isufi, and G.~Leus, ``Online time-varying topology
  identification via prediction-correction algorithms,'' in \emph{ICASSP
  2021-2021 IEEE International Conference on Acoustics, Speech and Signal
  Processing (ICASSP)}.\hskip 1em plus 0.5em minus 0.4em\relax IEEE, 2021, pp.
  5400--5404.

\bibitem{coutino2019advances}
M.~Coutino, E.~Isufi, and G.~Leus, ``Advances in distributed graph filtering,''
  \emph{IEEE Transactions on Signal Processing}, vol.~67, no.~9, pp.
  2320--2333, 2019.

\bibitem{mateos2019connecting}
G.~Mateos, S.~Segarra, A.~G. Marques, and A.~Ribeiro, ``Connecting the dots:
  Identifying network structure via graph signal processing,'' \emph{IEEE
  Signal Processing Magazine}, vol.~36, no.~3, pp. 16--43, 2019.

\bibitem{dong2019learning}
X.~Dong, D.~Thanou, M.~Rabbat, and P.~Frossard, ``Learning graphs from data: A
  signal representation perspective,'' \emph{IEEE Signal Processing Magazine},
  vol.~36, no.~3, pp. 44--63, 2019.

\bibitem{kim2014inference}
Y.~Kim, S.~Han, S.~Choi, and D.~Hwang, ``Inference of dynamic networks using
  time-course data,'' \emph{Briefings in bioinformatics}, vol.~15, no.~2, pp.
  212--228, 2014.

\bibitem{mantegna1999hierarchical}
R.~N. Mantegna, ``Hierarchical structure in financial markets,'' \emph{The
  European Physical Journal B-Condensed Matter and Complex Systems}, vol.~11,
  no.~1, pp. 193--197, 1999.

\bibitem{sandryhaila2013discrete}
A.~{Sandryhaila} and J.~M.~F. {Moura}, ``Discrete signal processing on
  graphs,'' \emph{IEEE Transactions on Signal Processing}, vol.~61, no.~7, pp.
  1644--1656, April 2013.

\bibitem{kalofolias2017learningtvgraphs}
V.~{Kalofolias}, A.~{Loukas}, D.~{Thanou}, and P.~{Frossard}, ``Learning time
  varying graphs,'' in \emph{2017 IEEE International Conference on Acoustics,
  Speech and Signal Processing (ICASSP)}, 2017, pp. 2826--2830.

\bibitem{yamada2020time}
K.~Yamada, Y.~Tanaka, and A.~Ortega, ``Time-varying graph learning with
  constraints on graph temporal variation,'' \emph{arXiv preprint
  arXiv:2001.03346}, 2020.

\bibitem{hallac2017network}
D.~Hallac, Y.~Park, S.~Boyd, and J.~Leskovec, ``Network inference via the
  time-varying graphical lasso,'' in \emph{Proceedings of the 23rd ACM SIGKDD
  International Conference on Knowledge Discovery and Data Mining}, 2017, pp.
  205--213.

\bibitem{friedman2008sparse}
J.~Friedman, T.~Hastie, and R.~Tibshirani, ``Sparse inverse covariance
  estimation with the graphical lasso,'' \emph{Biostatistics}, vol.~9, no.~3,
  pp. 432--441, 2008.

\bibitem{baingana2017cascades}
B.~{Baingana} and G.~B. {Giannakis}, ``Tracking switched dynamic network
  topologies from information cascades,'' \emph{IEEE Transactions on Signal
  Processing}, vol.~65, no.~4, pp. 985--997, 2017.

\bibitem{money2021online}
R.~Money, J.~Krishnan, and B.~Beferull-Lozano, ``Online non-linear topology
  identification from graph-connected time series,'' \emph{arXiv preprint
  arXiv:2104.00030}, 2021.

\bibitem{giannakis2018topology}
G.~B. Giannakis, Y.~Shen, and G.~V. Karanikolas, ``Topology identification and
  learning over graphs: Accounting for nonlinearities and dynamics,''
  \emph{Proceedings of the IEEE}, vol. 106, no.~5, pp. 787--807, 2018.

\bibitem{vlaski2018online}
S.~{Vlaski}, H.~P. {Maretić}, R.~{Nassif}, P.~{Frossard}, and A.~H. {Sayed},
  ``Online graph learning from sequential data,'' in \emph{2018 IEEE Data
  Science Workshop (DSW)}, 2018, pp. 190--194.

\bibitem{shafipour2020online}
\BIBentryALTinterwordspacing
R.~Shafipour and G.~Mateos, ``Online topology inference from streaming
  stationary graph signals with partial connectivity information,''
  \emph{Algorithms}, vol.~13, no.~9, 2020. [Online]. Available:
  \url{https://www.mdpi.com/1999-4893/13/9/228}
\BIBentrySTDinterwordspacing

\bibitem{marques2017stationary}
A.~G. Marques, S.~Segarra, G.~Leus, and A.~Ribeiro, ``Stationary graph
  processes and spectral estimation,'' \emph{IEEE Transactions on Signal
  Processing}, vol.~65, no.~22, pp. 5911--5926, 2017.

\bibitem{zaman2020online}
B.~{Zaman}, L.~M.~L. {Ramos}, D.~{Romero}, and B.~{Beferull-Lozano}, ``Online
  topology identification from vector autoregressive time series,'' \emph{IEEE
  Transactions on Signal Processing}, vol.~69, pp. 210--225, 2021.

\bibitem{saboksayr2021online}
S.~S. Saboksayr, G.~Mateos, and M.~Cetin, ``Online discriminative graph
  learning from multi-class smooth signals,'' \emph{Signal Processing}, vol.
  186, p. 108101, 2021.

\bibitem{simonetto2020time}
A.~Simonetto, E.~Dall'Anese, S.~Paternain, G.~Leus, and G.~B. Giannakis,
  ``Time-varying convex optimization: Time-structured algorithms and
  applications,'' \emph{Proceedings of the IEEE}, 2020.

\bibitem{shuman2013emerging}
D.~I. Shuman, S.~K. Narang, P.~Frossard, A.~Ortega, and P.~Vandergheynst, ``The
  emerging field of signal processing on graphs: Extending high-dimensional
  data analysis to networks and other irregular domains,'' \emph{IEEE signal
  processing magazine}, vol.~30, no.~3, pp. 83--98, 2013.

\bibitem{dempster1972covariance}
A.~P. Dempster, ``Covariance selection,'' \emph{Biometrics}, pp. 157--175,
  1972.

\bibitem{ullman2003structural}
J.~B. Ullman and P.~M. Bentler, ``Structural equation modeling,''
  \emph{Handbook of psychology}, pp. 607--634, 2003.

\bibitem{kalofolias2016learn}
V.~Kalofolias, ``How to learn a graph from smooth signals,'' in
  \emph{Artificial Intelligence and Statistics}.\hskip 1em plus 0.5em minus
  0.4em\relax PMLR, 2016, pp. 920--929.

\bibitem{kumar2020unified}
S.~Kumar, J.~Ying, J.~V. de~Miranda~Cardoso, and D.~P. Palomar, ``A unified
  framework for structured graph learning via spectral constraints.''
  \emph{Journal of Machine Learning Research}, vol.~21, no.~22, pp. 1--60,
  2020.

\bibitem{simonetto2016class}
A.~Simonetto, A.~Mokhtari, A.~Koppel, G.~Leus, and A.~Ribeiro, ``A class of
  prediction-correction methods for time-varying convex optimization,''
  \emph{IEEE Transactions on Signal Processing}, vol.~64, no.~17, pp.
  4576--4591, 2016.

\bibitem{martinet1970regularisation}
B.~Martinet, ``R{\'e}gularisation d’in{\'e}quations variationnelles par
  approximations successives. rev. fran{\c{c}}aise informat,'' \emph{Recherche
  Op{\'e}rationnelle}, vol.~4, pp. 154--158, 1970.

\bibitem{vial1983strong}
J.-P. Vial, ``Strong and weak convexity of sets and functions,''
  \emph{Mathematics of Operations Research}, vol.~8, no.~2, pp. 231--259, 1983.

\bibitem{magnus2019matrix}
J.~R. Magnus and H.~Neudecker, \emph{Matrix differential calculus with
  applications in statistics and econometrics}.\hskip 1em plus 0.5em minus
  0.4em\relax John Wiley \& Sons, 2019.

\bibitem{bastianello2020primal}
N.~Bastianello, A.~Simonetto, and R.~Carli, ``Primal and dual
  prediction-correction methods for time-varying convex optimization,'' 2020.

\bibitem{ryu2016primer}
E.~K. Ryu and S.~Boyd, ``Primer on monotone operator methods,'' \emph{Appl.
  Comput. Math}, vol.~15, no.~1, pp. 3--43, 2016.

\bibitem{combettes2011proximal}
P.~L. Combettes and J.-C. Pesquet, ``Proximal splitting methods in signal
  processing,'' in \emph{Fixed-point algorithms for inverse problems in science
  and engineering}.\hskip 1em plus 0.5em minus 0.4em\relax Springer, 2011, pp.
  185--212.

\bibitem{zhan2005extremal}
X.~Zhan, ``Extremal eigenvalues of real symmetric matrices with entries in an
  interval,'' \emph{SIAM journal on matrix analysis and applications}, vol.~27,
  no.~3, pp. 851--860, 2005.

\bibitem{das2008sharp}
K.~C. Das and R.~Bapat, ``A sharp upper bound on the spectral radius of
  weighted graphs,'' \emph{Discrete mathematics}, vol. 308, no.~15, pp.
  3180--3186, 2008.

\bibitem{beck2017first}
A.~Beck, \emph{First-order methods in optimization}.\hskip 1em plus 0.5em minus
  0.4em\relax SIAM, 2017.

\bibitem{grant2008cvx}
M.~Grant, S.~Boyd, and Y.~Ye, ``Cvx: Matlab software for disciplined convex
  programming,'' 2008.

\bibitem{perraudin2014gspbox}
N.~{Perraudin}, J.~{Paratte}, D.~{Shuman}, L.~{Martin}, V.~{Kalofolias},
  P.~{Vandergheynst}, and D.~K. {Hammond}, ``{GSPBOX: A toolbox for signal
  processing on graphs},'' \emph{ArXiv e-prints}, Aug. 2014.

\bibitem{dong2016learning}
X.~Dong, D.~Thanou, P.~Frossard, and P.~Vandergheynst, ``Learning laplacian
  matrix in smooth graph signal representations,'' \emph{IEEE Transactions on
  Signal Processing}, vol.~64, no.~23, pp. 6160--6173, 2016.

\bibitem{yahoo}
``Yahoo! finance,'' \url{https://finance.yahoo.com/lookup?s=API }, [Online;
  accessed July-2021].

\bibitem{kramer2008emergent}
M.~A. Kramer, E.~D. Kolaczyk, and H.~E. Kirsch, ``Emergent network topology at
  seizure onset in humans,'' \emph{Epilepsy research}, vol.~79, no. 2-3, pp.
  173--186, 2008.

\bibitem{bauschke2011convex}
H.~H. Bauschke, P.~L. Combettes \emph{et~al.}, \emph{Convex analysis and
  monotone operator theory in Hilbert spaces}.\hskip 1em plus 0.5em minus
  0.4em\relax Springer, 2011, vol. 408.

\bibitem{combettes2005signal}
P.~L. Combettes and V.~R. Wajs, ``Signal recovery by proximal forward-backward
  splitting,'' \emph{Multiscale Modeling \& Simulation}, vol.~4, no.~4, pp.
  1168--1200, 2005.

\end{thebibliography}

\end{document}